\documentclass{article}

\usepackage[final]{neurips_2022}

\usepackage[utf8]{inputenc} 
\usepackage[T1]{fontenc}    
\usepackage{hyperref}       
\usepackage{url}            
\usepackage{booktabs}       
\usepackage{amsfonts}       
\usepackage{nicefrac}       
\usepackage{microtype}      
\usepackage{xcolor}         
\usepackage{amsmath,commath}
\usepackage{graphicx}
\usepackage{caption}
\usepackage{subcaption}
\usepackage{multirow}
\usepackage{diagbox}
\usepackage{amsthm}
\usepackage{float}
\usepackage{algorithm}
\usepackage{algpseudocode}

\usepackage{enumitem}

\usepackage{wrapfig}

\usepackage{tikz}
\usepackage{pgfplots}
	
\usepackage{soul}
\pgfplotsset{compat=newest}

\newtheorem{theorem}{Theorem}
\newtheorem{lemma}{Lemma}
\newtheorem{assumption}{Assumption}
\newtheorem{proposition}{Proposition}

\newtheorem{definition}{Definition}

\title{Improved Bounds on Neural Complexity for Representing Piecewise Linear Functions}

\author{%
  Kuan-Lin Chen \\
  Department of Electrical and Computer Engineering\\
  University of California, San Diego\\
  La Jolla, CA 92093, USA \\
  \texttt{kuc029@ucsd.edu} \\
  \And
  Harinath Garudadri \\
  Qualcomm Institute\\
  University of California, San Diego\\
  La Jolla, CA 92093, USA \\
  \texttt{hgarudadri@ucsd.edu} \\
  \AND
  Bhaskar D. Rao \\
  Department of Electrical and Computer Engineering\\
  University of California, San Diego\\
  La Jolla, CA 92093, USA \\
  \texttt{brao@ucsd.edu} \\
}

\begin{document}

\maketitle

\begin{abstract}
A deep neural network using rectified linear units represents a continuous piecewise linear (CPWL) function and vice versa. Recent results in the literature estimated that the number of neurons needed to exactly represent any CPWL function grows exponentially with the number of pieces or exponentially in terms of the factorial of the number of distinct linear components. Moreover, such growth is amplified linearly with the input dimension. These existing results seem to indicate that the cost of representing a CPWL function is expensive. In this paper, we propose much tighter bounds and establish a polynomial time algorithm to find a network satisfying these bounds for any given CPWL function. We prove that the number of hidden neurons required to exactly represent any CPWL function is at most a quadratic function of the number of pieces. In contrast to all previous results, this upper bound is invariant to the input dimension. Besides the number of pieces, we also study the number of distinct linear components in CPWL functions. When such a number is also given, we prove that the quadratic complexity turns into bilinear, which implies a lower neural complexity because the number of distinct linear components is always not greater than the minimum number of pieces in a CPWL function. When the number of pieces is unknown, we prove that, in terms of the number of distinct linear components, the neural complexities of any CPWL function are at most polynomial growth for low-dimensional inputs and factorial growth for the worst-case scenario, which are significantly better than existing results in the literature.
\end{abstract}

\section{Introduction}
The rectified linear unit (ReLU) \citep{fukushima1980neocognitron,nair2010rectified} activation has been by far the most widely used nonlinearity and successful building block in deep neural networks (DNNs). Numerous architectures based on ReLU DNNs have achieved remarkable performance or state-of-the-art accuracy in speech processing \citep{zeiler2013rectified,maas2013rectifier}, computer vision \citep{krizhevsky2012imagenet,simonyan2015very,he2016deep}, medical image segmentation \citep{ronneberger2015u}, game playing \citep{mnih2015human,silver2016mastering}, and natural language processing \citep{vaswani2017attention}, just to name a few.
Besides such unprecedented empirical success, ReLU DNNs are also probably the most understandable nonlinear deep learning models due to their ability to be ``un-rectified'' \citep{hwang2019rectifying}.

The ability to demystify ReLU DNNs via ``un-rectifying ReLUs'' dates back to a seminal work by \citeauthor{pascanu2013number} in \citeyear{pascanu2013number}. Because each of ReLUs in a hidden layer divides the space of the preceding layer's output into two half spaces whose ReLU response is affine in one half space and exactly zero in the other, the layer of ReLUs can be replaced by an input-dependent diagonal matrix whose diagonal elements are ones for firing ReLUs and zeros for non-firing ReLUs. Based on this rationale, \cite{pascanu2013number} proved that a neural network using ReLUs divides the input space into many linear regions such that the network itself is an affine function within every region. Two excellent visualizations are shown in Figure 2 in \citep{hanin2019complexity} and \citep{hanin2019deep}. At this point, it is quite evident that any ReLU network exactly represents a CPWL function. \citeauthor{pascanu2013number} also proved that the maximum number of linear regions for any ReLU network with a single hidden layer is equivalent to the number of connected components induced by arrangements of hyperplanes in general position where each hyperplane corresponds to a ReLU in the hidden layer. Such a number can be computed in a closed form by Zaslavsky's Theorem \citep{zaslavsky1975facing}. Furthermore, they showed that the maximum number of linear regions can be bounded from below by exponential growth in terms of the number of hidden layers, leading to a conclusion that ReLU DNNs can generate more linear regions than their shallow counterparts. In the same year, \citeauthor{montufar2014number} improved such a lower bound and gave the first upper bound for the maximum number of linear regions. These bounds and their assumptions were later improved by \citep{montufar2017notes,raghu2017expressive,arora2018understanding,serra2018bounding,hinz2019framework}, just to name a few. We refer readers to \citeauthor{hinz2021analysis}'s doctoral thesis for a thorough discussion on the upper bound of the number of linear regions. Because a CPWL function with more pieces can better approximate any given continuous function and a ReLU DNN exactly represents a CPWL function \citep{arora2018understanding}, a ReLU DNN with more linear regions in general exhibits stronger expressivity. In summary, this ``un-rectifying'' perspective provides us a new angle to understand ReLU DNNs, and the results in some ways align with advances in approximation theory demonstrating the expressivity.\footnote{The approximation viewpoint is not the focus of this paper. The literature on approximation is vast and we refer readers to \citep{vardi2021size,lu2017expressive,eldan2016power,telgarsky2016benefits,hornik1989multilayer,cybenko1989approximation}, just to name a few.}

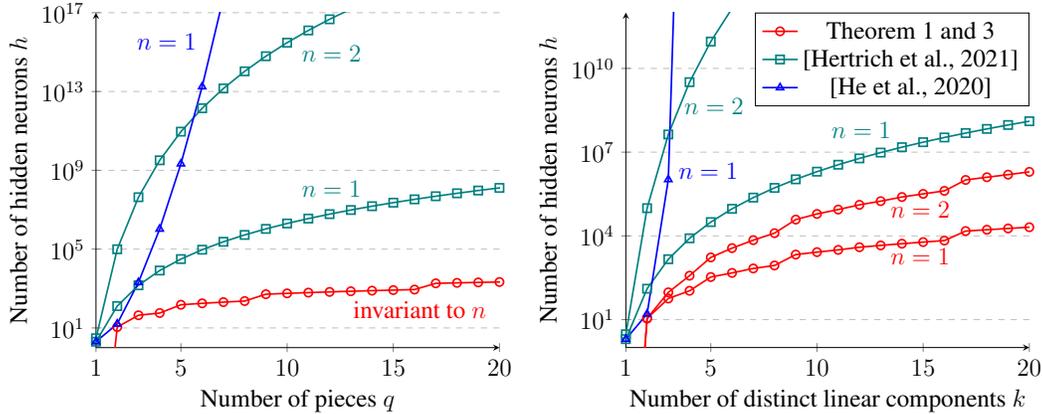
\begin{figure}[tb]
\centering
\resizebox{\columnwidth}{!}
{
\begin{tikzpicture}
\draw[very thick,teal] (4.0,2.7) node {\large $n=1$};
\draw[very thick,teal] (4.0,5.0) node {\large $n=2$};
\draw[very thick,blue] (1.2,5.2) node {\large $n=1$};
\draw[very thick,red] (5.5,0.65) node {\large invariant to $n$};
\begin{axis}[
    font=\large,
    axis lines = left,
    xlabel = Number of pieces $q$,
    ylabel = Number of hidden neurons $h$,
    xtick={1,5,10,15,20},
    xmax = 20,
    ymajorgrids=true,
    grid style=dashed,
    ymin = 1,
    ymax= 100000000000000000,
    ymode=log
]
\addplot [
    color=red,
    mark=o,
    thick,
]
coordinates {(1,0.000000001)
(2,11)
(3,44)
(4,57)
(5,150)
(6,177)
(7,204)
(8,231)
(9,514)
(10,567)
(11,620)
(12,673)
(13,726)
(14,779)
(15,832)
(16,885)
(17,1834)
(18,1937)
(19,2040)
(20,2143)
(21,2246)
(22,2349)
(23,2452)
(24,2555)
(25,2658)
(26,2761)
(27,2864)
(28,2967)
(29,3070)
(30,3173)
(31,3276)
(32,3379)
(33,6810)
(34,7011)
(35,7212)
(36,7413)
(37,7614)
(38,7815)
(39,8016)
(40,8217)};

\addplot [
    color=teal,
    mark=square,
    thick
    ]
    coordinates {(1,2)
(2,128)
(3,1458)
(4,8192)
(5,31250)
(6,93312)
(7,235298)
(8,524288)
(9,1062882)
(10,2000000)
(11,3543122)
(12,5971968)
(13,9653618)
(14,15059072)
(15,22781250)
(16,33554432)
(17,48275138)
(18,68024448)
(19,94091762)
(20,128000000)
(21,171532242)
(22,226759808)
(23,296071778)
(24,382205952)
(25,488281250)
(26,617831552)
(27,774840978)
(28,963780608)
(29,1189646642)
(30,1458000000)
(31,1775007362)
(32,2147483648)
(33,2582935938)
(34,3089608832)
(35,3676531250)
(36,4353564672)
(37,5131452818)
(38,6021872768)
(39,7037487522)
(40,8192000000)};

\addplot [
    color=teal,
    mark=square,
    thick
    ]
    coordinates {(1,3)
(2,98304)
(3,43046721)
(4,3221225472)
(5,91552734375)
(6,1410554953728)
(7,14242684529829)
(8,105553116266496)
(9,617673396283947)
(10,3000000000000000)
(11,12531744508246952)
(12,46221064723759104)
(13,153557679042272256)
(14,466704286673436672)
(15,1313681671142578176)
(16,3458764513820540928)
(17,8587269154529447936)
(18,20239921849432375296)
(19,45543381089624391680)
(20,98304000000000000000)
(21,204366955748855054336)
(22,410640204046236123136)
(23,799905706393173753856)
(24,1514571848868138319872)
(25,2793967723846435864576)
(26,5031778026857177284608)
(27,8862938119652501356544)
(28,15292966065715172868096)
(29,25887566242794551902208)
(30,43046721000000001671168)
(31,70395785975534054277120)
(32,113336795588871485128704)
(33,179816836496596251181056)
(34,281387635655620949966848)
(35,434652237848785367072768)
(36,663221759162200073699328)
(37,1000338803855445989523456)
(38,1492365511544812066570240)
(39,2203384855713413867765760)
(40,3221225472000000000000000)};

\addplot [
    color=blue,
    mark=triangle,
    thick
    ]
    coordinates
    {(1,2)
(2,16)
(3,2048)
(4,1048576)
(5,2147483648)
(6,17592186044416)
(7,576460752303423488)
(8,75557863725914323419136)
(9,39614081257132168796771975168)
(10,83076749736557242056487941267521536)
(11,696898287454081973172991196020261297061888)
(12,23384026197294446691258957323460528314494920687616)
(13,3138550867693340381917894711603833208051177722232017256448)
(14,1684996666696914987166688442938726917102321526408785780068975640576)
(15,3618502788666131106986593281521497120414687020801267626233049500247285301248)
(16,31082702275611665134711390509176302506278509424834232340028998555822468563283335970816)
(17,1067993517960455041197510853084776057301352261178326384973520803911109862890320275011481043468288)};
\end{axis}
\draw[very thick,blue] (10.4,3.0) node {\large $n=1$};
\draw[very thick,teal] (13.0,3.7) node {\large $n=1$};
\draw[very thick,teal] (10.5,4.1) node {\large $n=2$};
\draw[very thick,red] (14.0,1.55) node {\large $n=1$};
\draw[very thick,red] (14.0,2.35) node {\large $n=2$};
\begin{axis}[
    font=\large,
    at={(9cm,0cm)},
    axis lines = left,
    xlabel = Number of distinct linear components $k$,
    ylabel = Number of hidden neurons $h$,
    legend style={at={(1.0,1.0)},anchor=north east},
    ymode=log,
    xmax = 20,
    xtick={1,5,10,15,20},
    ymajorgrids=true,
    grid style=dashed,
    ymin = 1,
    ymax= 1000000000000
]

\addplot [
    color=red,
    mark=o,
    thick,
    forget plot
]
coordinates {(1,0.000000001)
(2,11)
(3,57)
(4,108)
(5,336)
(6,471)
(7,681)
(8,870)
(9,2142)
(10,2619)
(11,3149)
(12,3924)
(13,4560)
(14,5249)
(15,5991)
(16,6786)
(17,14866)
(18,16617)
(19,18471)
(20,20428)
(21,22488)};

\addplot [
    color=red,
    mark=o,
    thick
]
coordinates {(1,0.000000001)
(2,11)
(3,95)
(4,375)
(5,1695)
(6,3642)
(7,7023)
(8,12516)
(9,38412)
(10,61041)
(11,87806)
(12,129513)
(13,175623)
(14,246476)
(15,319563)
(16,409398)
(17,1008790)
(18,1262685)
(19,1563960)
(20,1967329)
(21,2380359)};

\addlegendentry{Theorem \ref{theorem:ReLu_network_upper_bound_q} and \ref{theorem:ReLu_network_upper_bound_k}}

\addplot [
    color=teal,
    mark=square,
    thick,
    forget plot
    ]
    coordinates {
    (1,2)
(2,128)
(3,1458)
(4,8192)
(5,31250)
(6,93312)
(7,235298)
(8,524288)
(9,1062882)
(10,2000000)
(11,3543122)
(12,5971968)
(13,9653618)
(14,15059072)
(15,22781250)
(16,33554432)
(17,48275138)
(18,68024448)
(19,94091762)
(20,128000000)
(21,171532242)
(22,226759808)
(23,296071778)
(24,382205952)
(25,488281250)
(26,617831552)
(27,774840978)
(28,963780608)
(29,1189646642)
(30,1458000000)
(31,1775007362)
(32,2147483648)
(33,2582935938)
(34,3089608832)
(35,3676531250)
(36,4353564672)
(37,5131452818)
(38,6021872768)
(39,7037487522)
(40,8192000000)};

\addplot [
    color=teal,
    mark=square,
    thick
    ]
    coordinates {
    (1,3)
(2,98304)
(3,43046721)
(4,3221225472)
(5,91552734375)
(6,1410554953728)
(7,14242684529829)
(8,105553116266496)
(9,617673396283947)
(10,3000000000000000)
(11,12531744508246952)
(12,46221064723759104)
(13,153557679042272256)
(14,466704286673436672)
(15,1313681671142578176)
(16,3458764513820540928)
(17,8587269154529447936)
(18,20239921849432375296)
(19,45543381089624391680)
(20,98304000000000000000)
(21,204366955748855054336)
(22,410640204046236123136)
(23,799905706393173753856)
(24,1514571848868138319872)
(25,2793967723846435864576)
(26,5031778026857177284608)
(27,8862938119652501356544)
(28,15292966065715172868096)
(29,25887566242794551902208)
(30,43046721000000001671168)
(31,70395785975534054277120)
(32,113336795588871485128704)
(33,179816836496596251181056)
(34,281387635655620949966848)
(35,434652237848785367072768)
(36,663221759162200073699328)
(37,1000338803855445989523456)
(38,1492365511544812066570240)
(39,2203384855713413867765760)
(40,3221225472000000000000000)};

\addlegendentry{\citep{hertrich2021towards}}

\addplot [
    color=blue,
    mark=triangle,
    thick
    ]
    coordinates {
    (1,2)
(2,16)
(3,1048576)
(4,1267650600228229401496703205376)};
\addlegendentry{\citep{he2020relu}}

\end{axis}
\end{tikzpicture}}
\caption{Any CPWL function $\mathbb{R}^n\to\mathbb{R}$ with $q$ pieces or $k$ distinct linear components can be exactly represented by a ReLU network with at most $h$ hidden neurons. In Theorem \ref{theorem:ReLu_network_upper_bound_q} and \ref{theorem:ReLu_network_upper_bound_k}, $h=0$ when $q=1$ or $k=1$. The bounds in Theorem \ref{theorem:ReLu_network_upper_bound_q} and the worst-case bounds in Theorem \ref{theorem:ReLu_network_upper_bound_k} are invariant to $n$.  (\ref{eq:inferred_upper_bound}) is used to infer $h$ based on the depth and width given by \cite{hertrich2021towards}.
The upper bounds given by Theorem \ref{theorem:ReLu_network_upper_bound_q} and \ref{theorem:ReLu_network_upper_bound_k} are substantially lower than existing bounds in the literature, implying that any CPWL function can be exactly realized by a ReLU network at a much lower cost.
}
\label{fig:bounds}
\end{figure}

Despite these advancements in linear regions, the complexity of a ReLU DNN that exactly represents a given CPWL function remains largely unexplored. One can find that this question is the opposite direction of the above-mentioned line of research. Although \cite{arora2018understanding} proved that any CPWL function can be exactly represented by a ReLU DNN with a bounded depth, any estimates regarding the width or number of neurons of such a network were not given. The resources required for a ReLU neural network to exactly represent a CPWL function remained unknown until \cite{he2020relu} provided a bound to the complexity of a ReLU network that realizes any given CPWL function. They proved that the number of neurons is bounded from above by exponential growth in terms of the product between the number of pieces and the number of distinct linear components of a given CPWL function. Such an exponential bound also grows linearly with the input dimension. Because the number of pieces is an upper bound of the number of distinct linear components for any CPWL function \citep{tarela1999region,he2020relu}, the bound grows exponentially with the quadratic number of pieces, which seems to imply that the cost for representing a CPWL function by a ReLU DNN is exceedingly high.

The most recent upper bound can be inferred from a recent work by \cite{hertrich2021towards} although the number of hidden neurons was not directly given. \cite{hertrich2021towards} proved a width bound in terms of the number of distinct linear components under the same depth used by \cite{arora2018understanding} and \cite{he2020relu}.\footnote{The number of ``affine pieces'' used by Theorem 4.4 in \citep{hertrich2021towards} should be interpreted as the number of distinct linear components to best reflect the upper bound for the maximum width. Such an interpretation of ``affine pieces'' is different from the convention used by \cite{pascanu2013number,montufar2014number,arora2018understanding}, \cite{hanin2019complexity}, and this work.}
In particular, they proved that the maximum width of a ReLU network that represents any given CPWL function can be polynomially bounded from above in terms of the number of distinct linear components. However, the order of such a polynomial is a quadratic function of the input dimension, which can be immensely large for a small number of pieces or linear components when the input dimension is large. This bound grows larger with the input dimension even though the underlying CPWL function is just a one-hidden-layer ReLU network using only one ReLU (see Figure \ref{fig:bounds} for the difference between $n=1$ and $n=2$ when $q=2$ or $k=2$).

In this paper, we provide improved bounds showing that any CPWL function can be represented by a ReLU DNN whose neural complexity is bounded from above by functions with much slower growth (see Figure \ref{fig:bounds}). Our results imply that one can exactly realize any given CPWL function by a ReLU network at a much lower cost. On the other hand, in addition to guaranteeing the existence of such a network, we also give a \textit{polynomial time} algorithm to exactly find a network satisfying our bounds. To the best of our knowledge, our results regarding the computational resource for a ReLU network, i.e., the number of hidden neurons, are the lowest upper bounds in the existing literature and the algorithm is the first tailored procedure to find a network representation from any given CPWL function. Key results and main contributions of this paper are highlighted below.

\subsection{Key results and contributions} \label{contributions}
    \paragraph{Quadratic bounds.} We prove that any CPWL function with $q$ pieces can be represented by a ReLU network whose number of hidden neurons is bounded from above by a quadratic function of $q$. We also give the corresponding upper bounds for the maximum width, i.e., the maximum number of neurons per hidden layer, and the number of layers for such a network. The maximum width is bounded from above by $\mathcal{O}(q^2)$ and the number of layers is bounded from above by a logarithmic function of $q$, i.e., $\mathcal{O}(\log_2 q)$. \textit{These bounds are invariant to the input dimension.} For any affine function, the upper bounds for the maximum width and the number of hidden neurons are zero.
    
    \paragraph{Further improvements on neural complexity.} When the number of distinct linear components $k$ of any CPWL function is given along with the number of pieces $q$, the quadratic bounds $\mathcal{O}(q^2)$ for the number of hidden neurons and the maximum width turn into \textit{bilinear} bounds of $k$ and $q$, i.e., $\mathcal{O}(kq)$. Such a change reduces the neural complexity because $k\leq q$, and $q$ can be much larger than $k$. Still, these bounds are independent of the input dimension.
    
    \paragraph{Finding a network satisfying bilinear bounds.} We establish a polynomial time algorithm that finds a ReLU network representing any given CPWL function. The network found by the algorithm satisfies the bilinear bounds on the number of hidden neurons and the maximum width, and the logarithmic bound on the number of layers. Note that such an algorithm also guarantees that one can always reverse-engineer at least one ReLU network from the function it computes. Compared to the general-purpose reverse-engineering algorithm proposed by \cite{rolnick2020reverse}, our algorithm specializes in the situation when pieces of a CPWL function are given.
    
    \paragraph{Improved bounds from a perspective of linear components.} When the number of pieces of a CPWL function is unknown and only the number of linear components $k$ is available, we prove that the number of hidden neurons and maximum width are bounded from above by factorial growth. More precisely, $\mathcal{O}\left(k\cdot k!\right)$. The number of layers is bounded from above by linearithmic growth, or $\mathcal{O}(k\log_2k)$. However, when the input dimension $n$ grows sufficiently slower than $k$, e.g., $\mathcal{O}\left(\sqrt{k}\right)$, then bounds for the number of hidden neurons and maximum width reduce to \textit{polynomial} growth functions of order $2n+1$; and the linearithmic growth reduces to $\mathcal{O}\left(n\log_2k\right)$ for the depth.
    
    \paragraph{A new approach to choosing the depth.} Instead of scaling the depth of a ReLU network with the input dimension \citep{arora2018understanding,he2020relu,hertrich2021towards}, we reveal that constructing a ReLU network whose depth is scaled with the number of pieces of the given CPWL function is more advantageous. Such a scaling turns out to be the key to deriving better upper bounds. This insight is provided by the max-min representation of CPWL functions \citep{tarela1990representation}. The importance of this scaling on the depth in ReLU networks has not been well recognized by existing bounds in the literature. We discuss implications of different representations in Section \ref{representations_CPWL_implications}.
\section{Preliminaries} \label{preliminaries}
Notation and definitions used in this paper are set up and clarified in this section. The set $\{1,2,\cdots,m\}$ is denoted by $[m]$. $\mathbb{I}\left[condition\right]$ is an indicator function that gives $1$ if the \textit{condition} is true, and $0$ otherwise. The CPWL function is defined by Definition \ref{def:continuous_piecewise_linear_fn} below.
\begin{definition} \label{def:continuous_piecewise_linear_fn}
A function $p\colon\mathbb{R}^n\to\mathbb{R}$ is said to be CPWL if there exists a finite number of closed subsets of $\mathbb{R}^n$, say $\{\mathcal{U}_i\}_{i\in[m]}$, such that (a) $\mathbb{R}^n=\bigcup_{i\in[m]}\mathcal{U}_i$; (b) $p$ is affine on $\mathcal{U}_i, \forall i\in[m]$.
\end{definition}
A family of closed \textit{convex} subsets, say $\{\mathcal{X}_i\}_{i\in[q]}$, satisfying Definition \ref{def:continuous_piecewise_linear_fn} is also referred to as a family of convex regions, affine pieces or simply \textit{pieces} for a CPWL function in this paper. Definition \ref{def:continuous_piecewise_linear_fn} follows the definition of CPWL functions by \cite{ovchinnikov2002max}. Notice that there are different definitions in the literature. For example, \cite{chua1988canonical} and \cite{arora2018understanding} defined a CPWL function on a finite number of polyhedral regions. However, their definitions are essentially the same as Definition \ref{def:continuous_piecewise_linear_fn} because any family of closed subsets satisfying Definition \ref{def:continuous_piecewise_linear_fn} can be decomposed into polyhedral regions. It is possible that some of the closed subsets satisfying Definition \ref{def:continuous_piecewise_linear_fn} are non-convex even though the number of them reaches the minimum (see Figure 2 in \citep{wang2005generalization}). The continuity is implied by Definition \ref{def:continuous_piecewise_linear_fn} due to the subsets being closed.

Because the goal of this paper is to bound the complexity of a ReLU DNN that exactly represents any given CPWL function, it is necessary to be able to measure the complexity of a CPWL function. The complexity of a CPWL function can be described using two different perspectives. One is the number of pieces $q$, which is the number of closed convex subsets satisfying Definition \ref{def:continuous_piecewise_linear_fn}.
Because this number has a minimum and any finite number above the minimum can be a valid $m$ in Definition \ref{def:continuous_piecewise_linear_fn}, the bounds become obviously loose when the number of pieces is not the minimum. Without loss of generality, we are interested in the number $q$ when it is the minimum.
The other is the number of distinct linear components $k$. A linear component of a CPWL function is defined in Definition \ref{def:linear_components}.
\begin{definition} \label{def:linear_components}
An affine function $f$ is said to be a linear component of a CPWL function $p$ if there exists a nonempty subset $\mathcal{M}\subseteq[m]$ such that $f(\mathbf{x})=p(\mathbf{x}), \forall \mathbf{x}\in\bigcup_{i\in\mathcal{M}}\mathcal{U}_i$ where $\{\mathcal{U}_i\}_{i\in[m]}$ is a family of the minimum number of closed subsets satisfying Definition \ref{def:continuous_piecewise_linear_fn}.
\end{definition}

A greater $q$ or $k$ gives a CPWL function more degrees of freedom because a CPWL function allowed to use $q+1$ pieces or $k+1$ arbitrary linear components can represent any CPWL function with $q$ pieces or $k$ distinct linear components and still have the flexibility to modify existing affine maps or increase the number of distinct affine maps of the CPWL function. Although increasing them both leads to a CPWL function with greater flexibility, the speed of upgrading degrees of freedom is different from each other. Note that a CPWL function with $q$ pieces can never have more than $q$ distinct linear components and a CPWL function with $k$ distinct linear components can easily have more than $k$ minimum number of pieces. Such a difference in a $1$-dimensional case can be clearly observed from Figure 1 in \citep{tarela1999region}. Note that it is possible for two disjoint subsets from a minimum number of closed subsets satisfying Definition \ref{def:continuous_piecewise_linear_fn} to share the same linear component. In other words, a linear component can be reused by multiple pieces. Hence, increasing $k$ gives faster growth than increasing $q$ for the complexity and expressivity of CPWL functions.

We define the ReLU activation function in Definition \ref{def:ReLU}. The ReLU network defined in Definition \ref{def:plain_ReLU_network} is a simple architecture which is usually referred to as a \textit{ReLU multi-layer perceptron}. Definition \ref{def:num_hidden_neurons_width} defines the corresponding number of hidden neurons, depth, and maximum width.
\begin{definition} \label{def:ReLU}
The rectified linear unit (ReLU) activation function $\sigma$ is defined as
$
        \sigma(x)=\max(0,x).
$
The ReLU layer or vector-valued rectified linear activation function $\sigma_k$ is defined as
$
    \sigma_k(\mathbf{x})=\begin{bmatrix}\sigma(x_1)&\sigma(x_2)&\cdots&\sigma(x_k)\end{bmatrix}^\mathsf{T}
$
where $\mathbf{x}=\begin{bmatrix}x_1&x_2&\cdots&x_k\end{bmatrix}^\mathsf{T}$.
\end{definition}
\begin{definition} \label{def:plain_ReLU_network}
    Let $l$ be any positive integer. A function $g\colon\mathbb{R}^{k_0}\to\mathbb{R}^{k_l}$ is said to be an $l$-layer ReLU network if there exist weights $\mathbf{W}_i\in\mathbb{R}^{k_{i}\times k_{i-1}}$ and $\mathbf{b}_i\in\mathbb{R}^{k_i}$ for $i\in[l]$ such that the input-output relationship of the network satisfies
    $
        g(\mathbf{x})=h_l(\mathbf{x})
    $
    where $h_1(\mathbf{x})=\mathbf{W}_1\mathbf{x}+\mathbf{b}_1$ and
    $
        h_i(\mathbf{x})=\mathbf{W}_{i}\sigma_{k_{i-1}}\left(h_{i-1}(\mathbf{x})\right)+\mathbf{b}_{i}
    $
    for every $i\in[l]\setminus [1]$.
\end{definition}
\begin{definition} \label{def:num_hidden_neurons_width}
    The sum $\sum_{l=1}^{L-1}k_l$ and the maximum $\max_{l\in[L-1]}k_l$ for $L>1$ are referred to as the number of hidden neurons and the maximum width of an $L$-layer ReLU network, respectively. Any $1$-layer ReLU network is said to have $0$ hidden neurons and a maximum width of $0$. An $l$-layer ReLU network is said to have depth $l$ and $l-1$ hidden layers.
\end{definition}

\section{Upper bounds on neural complexity for representing CPWL functions} \label{upper_bounds_on_neural_complexity}
The correspondence between CPWL functions and ReLU networks was first clearly confirmed by Theorem 2.1 in \citep{arora2018understanding} although a weaker version of the correspondence can be inferred from Proposition 4.1 in \citep{goodfellow2013maxout}. \cite{arora2018understanding} proved that every ReLU network $\mathbb{R}^n\to\mathbb{R}$ exactly represents a CPWL function, and the converse is also true, i.e., every CPWL function can be exactly represented by a ReLU network. One of the key steps used by \cite{arora2018understanding} to construct a ReLU network from any given CPWL function relies on an important representation result by \cite{wang2005generalization}, stating that any CPWL function can be represented by a sum of a finite number of \textit{max-$\eta$-affine} functions \citep{magnani2009convex} whose signs may be flipped and $\eta$ is bounded from above by $n+1$ where $\eta$ is the number of affine functions in the \textit{max-$\eta$-affine} function. The implication of using this representation is later discussed in Section \ref{constrained_depth} and its \textit{max-$\eta$-affine} functions are given therein. The bound $\eta\leq n+1$ in the representation allowed \cite{arora2018understanding} to further prove that there exists a ReLU DNN with at most
\begin{equation} \label{eq:prior_bound_on_depth}
    \left\lceil\log_2(n+1)\right\rceil
\end{equation}
hidden layers to exactly realize any given CPWL function. However, the computational resource required for a ReLU network to exactly represent any CPWL function had not been available in the literature until the work by \cite{he2020relu}.
\subsection{Upper bounds in prior work}
\cite{he2020relu} proved that a CPWL function $\mathbb{R}^n\to\mathbb{R}$ with $q$ pieces and $k$ linear components can be represented by a ReLU network whose number of neurons is given by
\begin{equation} \label{eq:he_bound}
    \begin{cases}
    \mathcal{O}\left(n2^{kq+(n+1)(k-n-1)}\right), &\text{if } k\geq n+1,\\
    \mathcal{O}\left(n2^{kq}\right), &\text{if } k<n+1.
    \end{cases}
\end{equation}
The number of hidden layers in such a ReLU DNN is also bounded from above by $\left\lceil\log_2(n+1)\right\rceil$, which is the same as the bound derived by \cite{arora2018understanding}. One of their significant contributions in our view is that they utilize the number of pieces and linear components of a CPWL function to bound the complexity of the equivalent ReLU network. \cite{he2020relu} also proved the relationship
\begin{equation} \label{eq:he_bound_k_q}
    k\leq q \leq k!
\end{equation}
for any CPWL function. Note that the bounds in (\ref{eq:he_bound_k_q}) on the number of pieces $q$ and linear components $k$ were first mentioned by \cite{tarela1999region} who developed the lattice representation of CPWL functions.
Asymptotically, the bounds in (\ref{eq:he_bound}) for $k\geq n+1$ and $k<n+1$ are amplified linearly with the input dimension $n$ for any fixed $k$.
Due to (\ref{eq:he_bound_k_q}), they can be further bounded from above by $\mathcal{O}\left(n2^{q^2+(n+1)(q-n-1)}\right)$ and $\mathcal{O}\left(n2^{q^2}\right)$ in terms of $q$ and $n$. On the other hand, in terms of $k$ and $n$, they can be further bounded from above by $\mathcal{O}\left(n2^{k\cdot k!+(n+1)(k-n-1)}\right)$ and $\mathcal{O}\left(n2^{k\cdot k!}\right)$. Because these bounds grow much faster than exponential growth, they seem to suggest that the cost of computing a CPWL function via a ReLU network is exceptionally high.

\cite{hertrich2021towards} proved that any CPWL function $\mathbb{R}^n\to\mathbb{R}$ with $k$ distinct linear components can be represented by a ReLU network whose maximum width is
$
    \mathcal{O}\left(k^{2n^2+3n+1}\right)
$
under the same number of hidden layers $\left\lceil\log_2(n+1)\right\rceil$. Hence, the number of hidden neurons must be bounded from above by
\begin{equation} \label{eq:inferred_upper_bound}
    \mathcal{O}\left(k^{2n^2+3n+1}\log_2\left(n+1\right)\right).
\end{equation}
Note that we infer this bound by taking the product of the depth and the maximum width. Using $k\leq q$,  the bound in (\ref{eq:inferred_upper_bound}) can be expressed in terms of $q$, leading to $\mathcal{O}\left(q^{2n^2+3n+1}\log_2\left(n+1\right)\right)$. Such a bound can grow slower than $\mathcal{O}\left(n2^{q^2}\right)$, but it grows faster than $\mathcal{O}\left(n2^{q^2}\right)$ if the input dimension $n$ grows sufficiently faster than the number of pieces $q$. Also,  $\mathcal{O}\left(k^{2n^2+3n+1}\log_2\left(n+1\right)\right)$ grows faster than $\mathcal{O}\left(n2^{k\cdot k!}\right)$ when the input dimension $n$ grows sufficiently faster than the number of distinct linear components $k$.
\subsection{Improved upper bounds} \label{improved_upper_bounds}
We show that any CPWL function can be represented by a ReLU network whose number of hidden neurons is bounded by much slower growth functions.
We state our main results in Theorem \ref{theorem:ReLu_network_upper_bound_q}, \ref{theorem:ReLu_network_upper_bound_k_q} and \ref{theorem:ReLu_network_upper_bound_k}, and focus on their impact in this subsection. Each one of them is tailored to a specific complexity measure of the CPWL function. Their proof sketches are deferred to Section \ref{proof_sketches}. We first focus on the case when the number of linear components is unknown and the complexity of the CPWL is only measured by the number of pieces $q$.

\begin{theorem} \label{theorem:ReLu_network_upper_bound_q}
Any CPWL function $p\colon\mathbb{R}^n\to\mathbb{R}$ with $q$ pieces can be represented by
a ReLU network whose number of layers $l$, maximum width $w$, and number of hidden neurons $h$ satisfy
\begin{equation}
    l\leq2\left\lceil\log_2q\right\rceil+1,
\end{equation}
\begin{equation} \label{eq:h_w}
    w\leq \mathbb{I}\left[q>1\right]\left\lceil\frac{3q}{2}\right\rceil q,
\end{equation}
and
\begin{equation} \label{eq:h_q}
    h\leq \left(3\cdot 2^{\left\lceil\log_2q\right\rceil}+2\left\lceil\log_2q\right\rceil-3\right)q+3\cdot 2^{\left\lceil\log_2q\right\rceil}-2\left\lceil\log_2q\right\rceil-3.
\end{equation}
Furthermore, Algorithm \ref{alg:find_a_network_for_a_piecewise_linear_fn} finds such a network in $\mathsf{poly}\left(n,q,L\right)$ time where $L$ is the number of bits required to represent every entry of the rational matrix $\mathbf{A}_i$ in the polyhedron representation $\{\mathbf{x}\in\mathbb{R}^n|\mathbf{A}_i\mathbf{x}\leq\mathbf{b}_i\}$ of the piece $\mathcal{X}_i$ for every $i\in[q]$.
\end{theorem}
\begin{algorithm}[H]
\caption{Find a ReLU network that computes a given continuous piecewise linear function}\label{alg:find_a_network_for_a_piecewise_linear_fn}
\begin{algorithmic}[1]
\Require A CPWL function $p$ with pieces $\{\mathcal{X}_i\}_{i\in[q]}$ of $\mathbb{R}^n$ satisfying Definition \ref{def:continuous_piecewise_linear_fn}.
\Ensure A ReLU network $g$ computing $g(\mathbf{x})=p(\mathbf{x}),\forall \mathbf{x}\in\mathbb{R}^n$.
\State $f_1,f_2,\cdots,f_k \gets$ run Algorithm \ref{alg:find_all_distinct_linear_components} to find all distinct linear components of $p$
\For{$i=1,2,\cdots,q$}
    \State $\mathcal{A}_i \gets \emptyset$
    \For{$j=1,2\cdots,k$}
        \If{$f_j(\mathbf{x})\geq p(\mathbf{x}),\forall \mathbf{x}\in\mathcal{X}_i$}
            \State $\mathcal{A}_i \gets \mathcal{A}_i\bigcup \{j\}$
        \EndIf
    \EndFor
    \State $v_i \gets $ run Algorithm \ref{alg:find_a_extremum_ReLU_network} with $\{f_m\}_{m\in\mathcal{A}_i}$ using the minimum type
\EndFor
\State $v \gets $ run Algorithm \ref{alg:find_a_concat_ReLU_network} with $v_1,v_2,\cdots,v_q$ \Comment{Combine $q$ ReLU networks in parallel}
\State $u \gets $ run Algorithm \ref{alg:find_a_extremum_ReLU_network} with $\left\{\begin{bmatrix}s_1&s_2&\cdots&s_q\end{bmatrix}^\mathsf{T}\mapsto s_m\right\}_{m\in[q]}$ using the maximum type
\State $g \gets $ run Algorithm \ref{alg:find_a_composited_ReLU_network} with $v$ and $u$ \Comment{Find a ReLU network for the composition $u\circ v$}
\end{algorithmic}
\end{algorithm}
The proof of Theorem \ref{theorem:ReLu_network_upper_bound_q} is deferred to Appendix \ref{proof:ReLu_network_upper_bound_q} in the supplementary material. Algorithm \ref{alg:find_all_distinct_linear_components}, \ref{alg:find_a_extremum_ReLU_network}, \ref{alg:find_a_concat_ReLU_network}, and \ref{alg:find_a_composited_ReLU_network} used by Algorithm \ref{alg:find_a_network_for_a_piecewise_linear_fn} are deferred to Appendix \ref{algorithm} in the supplementary material and will be discussed soon after the discussion on bounds. Because $2^{\left\lceil\log_2q\right\rceil}<2q$, the upper bound in (\ref{eq:h_q}) can be further bounded from above by
$
    6q^2+2\left\lceil\log_2q\right\rceil q+3q-2\left\lceil\log_2q\right\rceil-3,
$
leading to the asymptotic bound $h=\mathcal{O}(q^2)$. Obviously, $l=\mathcal{O}(\log_2 q)$ and $w=\mathcal{O}(q^2)$. Since the bound given by Theorem 5.2 in \cite{he2020relu} can be lower bounded by $\mathcal{O}\left(n2^{q^2}\right)$, it grows exponentially faster than our bound of $h$ given in Theorem \ref{theorem:ReLu_network_upper_bound_q}.
On the other hand, the upper bound given by (\ref{eq:inferred_upper_bound}) is at least polynomially larger than our bound of $h$ and the order of this polynomial grows quadratically with the input dimension $n$. Note that such a polynomial becomes an exponential function when the growth in $n$ is not slower than $q$.
Such differences are illustrated by the figure on the left-hand side of Figure \ref{fig:bounds}.
The bounds in Theorem \ref{theorem:ReLu_network_upper_bound_q} are independent of the input dimension $n$. Hence, one can realize any given CPWL function using a relatively small ReLU network even though $n$ is huge.

In terms of the maximum width, the upper bound given by (\ref{eq:h_w}) is at least polynomially smaller than the one given by \cite{hertrich2021towards}.
In contrast to the bound for the number of layers in \citep{arora2018understanding,he2020relu,hertrich2021towards} that grows logarithmically with the input dimension $n$, our bound in Theorem \ref{theorem:ReLu_network_upper_bound_q} grows logarithmically with the number of pieces $q$. Therefore, the ReLU network found by Algorithm \ref{alg:find_a_network_for_a_piecewise_linear_fn} in general becomes deeper when the CPWL becomes more complex for a fixed input dimension. On the other hand, the network remains the same depth even for an arbitrarily larger $n$ as long as $q$ is fixed.
Taking an affine function for example, a $1$-layer ReLU network with $0$ hidden neurons is the solution given by Theorem \ref{theorem:ReLu_network_upper_bound_q}. However, the bound given by \citep{arora2018understanding,he2020relu,hertrich2021towards} keeps increasing the depth for a larger $n$.

We briefly explain algorithms used by Algorithm \ref{alg:find_a_network_for_a_piecewise_linear_fn}. Algorithm \ref{alg:find_a_extremum_ReLU_network} finds a ReLU network that computes a \textit{max-affine} or \textit{min-affine} function \citep{magnani2009convex}. Algorithm \ref{alg:find_a_concat_ReLU_network} concatenates two given ReLU networks in parallel and returns another ReLU network computing the concatenation of two outputs. Algorithm \ref{alg:find_a_composited_ReLU_network} finds a ReLU network that represents a composition of two given ReLU networks. These algorithms are basic manipulations of ReLU networks. Algorithm \ref{alg:find_a_network_for_a_piecewise_linear_fn} is a polynomial time algorithm, following from the proof of Theorem \ref{theorem:ReLu_network_upper_bound_k_q}. Table \ref{tab:find_a_network_for_a_piecewise_linear_fn} in Appendix \ref{algorithm} in the supplementary material gives a complexity analysis for Algorithm \ref{alg:find_a_network_for_a_piecewise_linear_fn}.

Notice that Algorithm \ref{alg:find_a_network_for_a_piecewise_linear_fn} does not need to be given any linear components or completely know the CPWL function because every distinct linear component can be found by Algorithm \ref{alg:find_all_distinct_linear_components}, which only needs to be given a closed $\epsilon$-ball in the interior of every piece of a CPWL function $p$ and observe the output of $p$ when feeding an input. Algorithm \ref{alg:find_all_distinct_linear_components} solves a system of linear equations for every piece of $p$ to find the corresponding linear component. Every system of linear equations here has a unique solution because the interior of each of the pieces is nonempty. The nonemptyness is guaranteed by Lemma \ref{lemma:merge2}\ref{lemma:nonempty_interior} in Appendix \ref{lemmas} in the supplementary material.

The 5th step of Algorithm $\ref{alg:find_a_network_for_a_piecewise_linear_fn}$ can be executed by checking the optimization result of the following linear programming problem
\begin{equation} \label{eq:linear_programming_problem}
\begin{split}
    \text{minimize} & \ \ \  f_j(\mathbf{x})-p(\mathbf{x}),\\
    \text{subject to} & \ \ \ \mathbf{x}\in\mathcal{X}_i.
\end{split}
\end{equation}
The condition in the 5th step can only be true when the optimal value is nonnegative. Because every piece of $p$ is given to Algorithm $\ref{alg:find_a_network_for_a_piecewise_linear_fn}$, the piece $\mathcal{X}_i$ is available for the linear program as a system of linear inequalities. The objective function is also available since $p$ is affine on $\mathcal{X}_i$ and all distinct linear components are available from Algorithm \ref{alg:find_all_distinct_linear_components}. The corresponding linear component of $p$ on $\mathcal{X}_i$ can be found by first feeding at most $n+1$ affinely independent points from the closed $\epsilon$-ball to $p$ and every candidate linear component, and then matching their output values.

The ellipsoid method \citep{khachiyan1979polynomial}, the interior-point method \citep{karmarkar1984new}, and the path-following method \citep{renegar1988polynomial} are polynomial time algorithms for the linear programming problem using rational numbers on the Turing machine model of computation. These algorithms are also known to be \textit{weakly polynomial time} algorithms.
The strongly polynomial time algorithm requested by Smale's 9th problem \citep{smale1998mathematical}, i.e., \textit{the linear programming problem}, is still an open question. Given that we run the 5th step of Algorithm \ref{alg:find_a_network_for_a_piecewise_linear_fn} by solving the linear programming problem in (\ref{eq:linear_programming_problem}), Algorithm \ref{alg:find_a_network_for_a_piecewise_linear_fn} is a weakly polynomial time algorithm. The question of whether it is a strongly polynomial time algorithm is not known. The dependency on the number of bits $L$ in the time complexity of Algorithm \ref{alg:find_a_network_for_a_piecewise_linear_fn} directly comes from using (\ref{eq:linear_programming_problem}) to execute the 5th step.
In practice, linear programming problems can be solved very reliably and efficiently \citep{boyd2004convex}. We provide an implementation of Algorithm \ref{alg:find_a_network_for_a_piecewise_linear_fn} and measure its run time on a computer for different numbers of pieces and input dimensions in Appendix \ref{opensource_and_runtime} in the supplementary material.

Theorem \ref{theorem:ReLu_network_upper_bound_k_q} discusses the case when the number of linear components and pieces are both known.

\begin{theorem} \label{theorem:ReLu_network_upper_bound_k_q}
Any CPWL function $p\colon\mathbb{R}^n\to\mathbb{R}$ with $k$ linear components and $q$ pieces can be represented by a ReLU network whose number of layers $l$, maximum width $w$, and number of hidden neurons $h$ satisfy
$
    l\leq\left\lceil\log_2q\right\rceil+\left\lceil\log_2k\right\rceil+1,
$
$
    w\leq\mathbb{I}\left[k>1\right]\left\lceil\frac{3k}{2}\right\rceil q,
$
and
\begin{equation} \label{eq:thm_upper_bound_hidden_neurons}
        h\leq \left(3\cdot 2^{\left\lceil\log_2k\right\rceil}+2\left\lceil\log_2k\right\rceil-3\right)q+3\cdot 2^{\left\lceil\log_2q\right\rceil}-2\left\lceil\log_2k\right\rceil-3.
\end{equation}
Furthermore, Algorithm \ref{alg:find_a_network_for_a_piecewise_linear_fn} finds such a network in $\mathsf{poly}\left(n,k,q,L\right)$ time where $L$ is the number of bits required to represent every entry of the rational matrix $\mathbf{A}_i$ in the polyhedron representation $\{\mathbf{x}\in\mathbb{R}^n|\mathbf{A}_i\mathbf{x}\leq\mathbf{b}_i\}$ of the piece $\mathcal{X}_i$ for every $i\in[q]$.
\end{theorem}
The proof of Theorem \ref{theorem:ReLu_network_upper_bound_k_q} is deferred to Appendix \ref{proof:ReLu_network_upper_bound_k_q} in the supplementary material. The bounds in Theorem \ref{theorem:ReLu_network_upper_bound_k_q} are in general tighter and always no worse than those in Theorem \ref{theorem:ReLu_network_upper_bound_q} because $q$ is never less than $k$ but can be much larger than $k$. 
Asymptotically, $l=\mathcal{O}(\log_2 q)$, $w=\mathcal{O}(kq)$, and $h=\mathcal{O}(kq)$. The bound given by Theorem 5.2 in \cite{he2020relu} increases exponentially faster than the bound of $h$ in Theorem \ref{theorem:ReLu_network_upper_bound_k_q}.

When the number of linear components is the only complexity measure of the CPWL function, we resort to Theorem \ref{theorem:ReLu_network_upper_bound_k} below.
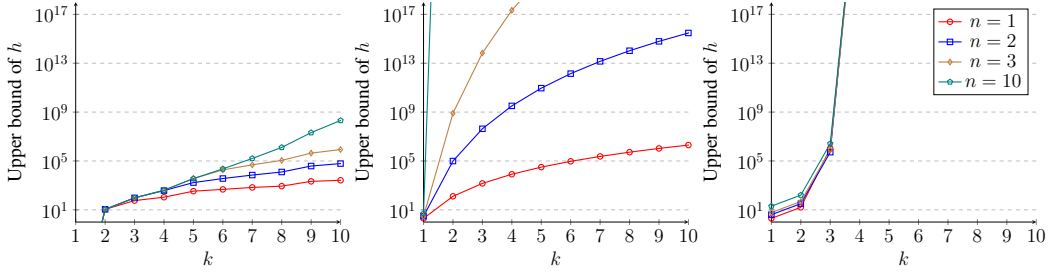
\begin{figure}[tb]
\centering
\resizebox{\columnwidth}{!}
{
    \begin{tikzpicture}
\begin{axis}[
    font=\Large,
    axis lines = left,
    xlabel=$k$,
    ylabel=Upper bound of $h$,
    xmin=1, xmax=20,
    ymin=1,
    xmax=10,
    xtick={1,2,3,4,5,6,7,8,9,10},
    legend pos=north west,
    ymajorgrids=true,
    grid style=dashed,
    ymode=log,
    ymax= 1000000000000000000
]
    
\addplot[
    color=red,
    mark=o,
    thick,
    ]
    coordinates {
(1,0.000000001)
(2,11)
(3,57)
(4,108)
(5,336)
(6,471)
(7,681)
(8,870)
(9,2142)
(10,2619)
(11,3149)
(12,3924)
(13,4560)
(14,5249)
(15,5991)
(16,6786)
(17,14866)
(18,16617)
(19,18471)
(20,20428)
(21,22488)
    };

\addplot[
    color=blue,
    mark=square,
    thick,
    ]
    coordinates {
(1,0.000000001)
(2,11)
(3,95)
(4,375)
(5,1695)
(6,3642)
(7,7023)
(8,12516)
(9,38412)
(10,61041)
(11,87806)
(12,129513)
(13,175623)
(14,246476)
(15,319563)
(16,409398)
(17,1008790)
(18,1262685)
(19,1563960)
(20,1967329)
(21,2380359)
    };

\addplot[
    color=brown,
    mark=diamond,
    thick,
    ]
    coordinates {
(1,0.000000001)
(2,11)
(3,95)
(4,401)
(5,3615)
(6,18615)
(7,48309)
(8,111720)
(9,438336)
(10,856119)
(11,1570421)
(12,2739113)
(13,4588579)
(14,7053821)
(15,11016799)
(16,16842206)
(17,44766622)
(18,64642689)
(19,88997451)
(20,124054021)
(21,165290071)
    };

\addplot[
    color=teal,
    mark=pentagon,
    thick,
    ]
    coordinates {
(1,0.000000001)
(2,11)
(3,95)
(4,401)
(5,3615)
(6,22503)
(7,160647)
(8,1285239)
(9,20805493)
(10,204909301)
(11,2316916981)
(12,26997697525)
(13,355801906165)
(14,5032766294005)
(15,75903811270645)
(16,1214460980330485)
(17,38324654954151928)
(18,165438784065949664)
(19,521384987029657472)
(20,1501219007434727936)
(21,4223067989500496896)
    };
\end{axis}

\begin{axis}[
    font=\Large,
    at={(9cm,0cm)},
    axis lines = left,
    xlabel = $k$,
    ylabel = Upper bound of $h$,
    legend pos=north east,
    ymode=log,
    xmax = 10,
    xtick={1,2,3,4,5,6,7,8,9,10},
    ymajorgrids=true,
    grid style=dashed,
    ymin = 1,
    ymax= 1000000000000000000
]
\addplot [
    color=red,
    mark=o,
    thick
]
coordinates {((1,2)
(2,128)
(3,1458)
(4,8192)
(5,31250)
(6,93312)
(7,235298)
(8,524288)
(9,1062882)
(10,2000000)
(11,3543122)
(12,5971968)
(13,9653618)
(14,15059072)
(15,22781250)
(16,33554432)
(17,48275138)
(18,68024448)
(19,94091762)
(20,128000000)
(21,171532242)
(22,226759808)
(23,296071778)
(24,382205952)
(25,488281250)
(26,617831552)
(27,774840978)
(28,963780608)
(29,1189646642)
(30,1458000000)
(31,1775007362)
(32,2147483648)
(33,2582935938)
(34,3089608832)
(35,3676531250)
(36,4353564672)
(37,5131452818)
(38,6021872768)
(39,7037487522)
(40,8192000000)};

\addplot [
    color=blue,
    mark=square,
    thick
]
coordinates {(1,3)
(2,98304)
(3,43046721)
(4,3221225472)
(5,91552734375)
(6,1410554953728)
(7,14242684529829)
(8,105553116266496)
(9,617673396283947)
(10,3000000000000000)
(11,12531744508246952)
(12,46221064723759104)
(13,153557679042272256)
(14,466704286673436672)
(15,1313681671142578176)
(16,3458764513820540928)
(17,8587269154529447936)
(18,20239921849432375296)
(19,45543381089624391680)
(20,98304000000000000000)
(21,204366955748855054336)
(22,410640204046236123136)
(23,799905706393173753856)
(24,1514571848868138319872)
(25,2793967723846435864576)
(26,5031778026857177284608)
(27,8862938119652501356544)
(28,15292966065715172868096)
(29,25887566242794551902208)
(30,43046721000000001671168)
(31,70395785975534054277120)
(32,113336795588871485128704)
(33,179816836496596251181056)
(34,281387635655620949966848)
(35,434652237848785367072768)
(36,663221759162200073699328)
(37,1000338803855445989523456)
(38,1492365511544812066570240)
(39,2203384855713413867765760)
(40,3221225472000000000000000)};

\addplot [
    color=brown,
    mark=diamond,
    thick
]
coordinates {(1,3)
(2,805306368)
(3,68630377364883)
(4,216172782113783808)
(5,111758708953857425408)
(6,18422826643394446491648)
(7,1379959609634220041830400)
(8,58028439341502200385896448)
(9,1570042899082081632228737024)
(10,30000000000000000948382466048)
(11,432629808319497703798154985472)
(12,4945339870828537631853131071488)
(13,46508798407987186985236420362240)
(14,370430087073743850133082498662400)
(15,2556680789771788017340478669193216)
(16,15576890575604482885591488987660288)
(17,85053277430603096670955741391093760)
(18,421455181554660564376545319341522944)
(19,1915235051777245775703586886125617152)
(20,8053063680000000254579479735999397888)
(21,31569049585831050307726079394448408576)
(22,116133179875436959810010665483849695232)};

\addplot [
    color=teal,
    mark=pentagon,
    thick
    ]
    coordinates {(1,5)
(2,17254365866976409468586889655692563631127772430425966387906310559498240)};

\end{axis}

\begin{axis}[
    font=\Large,
    at={(18cm,0cm)},
    axis lines = left,
    xlabel = $k$,
    ylabel = Upper bound of $h$,
    legend pos=north east,
    ymode=log,
    xmax = 10,
    xtick={1,2,3,4,5,6,7,8,9,10},
    ymajorgrids=true,
    grid style=dashed,
    ymin = 1,
    ymax= 1000000000000000000
]
\addplot [
    color=red,
    mark=o,
    thick
]
coordinates {(1,2)
(2,16)
(3,1048576)
(4,1267650600228229401496703205376)};
\addlegendentry{$n=1$}

\addplot [
    color=blue,
    mark=square,
    thick
]
coordinates {(1,4)
(2,32)
(3,524288)
(4,1267650600228229401496703205376)};
\addlegendentry{$n=2$}

\addplot [
    color=brown,
    mark=diamond,
    thick
]
coordinates {(1,6)
(2,48)
(3,786432)
(4,237684487542793012780631851008)};
\addlegendentry{$n=3$}

\addplot [
    color=teal,
    mark=pentagon,
    thick
    ]
    coordinates {(1,20)
(2,160)
(3,2621440)
(4,792281625142643375935439503360)};
\addlegendentry{$n=10$}

\end{axis}

\end{tikzpicture}}
  \caption{Left: The upper bound of $h$ in Theorem \ref{theorem:ReLu_network_upper_bound_k} grows much slower when $n$ grows sufficiently slower than $k$, leading to a much better upper bound compared to the worst-case asymptotic bound $\mathcal{O}\left(k\cdot k!\right)$ in Theorem \ref{theorem:ReLu_network_upper_bound_k} when $n$ is sufficiently larger. Middle: the bound in (\ref{eq:inferred_upper_bound}) inferred from \citep{hertrich2021towards}. Right: Theorem 5.2 in \citep{he2020relu}.}
  \label{fig:phi}
\end{figure}
\begin{theorem} \label{theorem:ReLu_network_upper_bound_k}
Any CPWL function $p\colon\mathbb{R}^n\to\mathbb{R}$ with $k$ linear components can be represented by
a ReLU network whose number of layers $l$, maximum width $w$, and number of hidden neurons $h$ satisfy
$
    l\leq\left\lceil\log_2\phi(n,k)\right\rceil+\left\lceil\log_2k\right\rceil+1,
$
$
    w\leq\mathbb{I}\left[k>1\right]\left\lceil\frac{3k}{2}\right\rceil \phi(n,k),
$
and
\begin{equation}
    h\leq \left(3\cdot 2^{\left\lceil\log_2k\right\rceil}+2\left\lceil\log_2k\right\rceil-3\right)\phi(n,k)+3\cdot 2^{\left\lceil\log_2\phi(n,k)\right\rceil}-2\left\lceil\log_2k\right\rceil-3
\end{equation}
where
\begin{equation}
    \phi(n,k)=\min\left(\sum_{i=0}^n \binom{\frac{k^2-k}{2}}{i},k!\right).
\end{equation}
\end{theorem}
The proof of Theorem \ref{theorem:ReLu_network_upper_bound_k} is deferred to Appendix \ref{proof:ReLu_network_upper_bound_k} in the supplementary material. Because $\phi(n,k)\leq k!$, the worst-case asymptotic bounds for $l$, $w$ and $h$ are $l=\mathcal{O}\left(k\log_2k\right)$, $w=\mathcal{O}\left(k\cdot k!\right)$, and $h=\mathcal{O}\left(k\cdot k!\right)$, respectively. However, it holds that
$
    \sum_{i=0}^n \binom{\frac{k^2-k}{2}}{i}\leq k^{2n},
$
so the asymptotic bounds are $l=\mathcal{O}\left(n\log_2k\right)$, $w=\mathcal{O}\left(k^{2n+1}\right)$, and $h=\mathcal{O}\left(k^{2n+1}\right)$ when $n$ grows sufficiently slower than $k$. For example, $n=\mathcal{O}\left(\sqrt{k}\right)$. In this case, $w$ and $h$ are bounded from above by a polynomial of order $2c\sqrt{k}+1$ for some constant $c$, which grows slower than factorial growth. Such an advantage for small $n$ is illustrated by the figure on the left-hand side of Figure \ref{fig:phi}.

Since the bound given by Theorem 5.2 in \citep{he2020relu} can be bounded from below by $\mathcal{O}\left(n2^{k\cdot k!}\right)$, it at the minimum grows exponentially larger than the upper bound of $h$ in Theorem \ref{theorem:ReLu_network_upper_bound_k}. Even for a small $n$, the relative order of growth is gigantic. The figure on the right-hand side of Figure \ref{fig:phi} illustrates such a large difference. For $k=5$, $n2^{k\cdot k!}\approx 7.92\times 10^{28}$ when $n=1$, while our bound of $h$ is at most $3615$ for any $n$. The difference is extremely large even though $k$ is small under $n=1$. The middle plot in Figure \ref{fig:phi} shows that (\ref{eq:inferred_upper_bound}) increases much faster when $n$ becomes larger. For $k=3$, $k^{2n^2+3n+1}\log_2\left(n+1\right)\approx 8.20\times 10^{110}$ when $n=10$, while our bound of $h$ is at most $95$ for any $n$. The upper bound of $h$ in Theorem \ref{theorem:ReLu_network_upper_bound_k} is much better than (\ref{eq:inferred_upper_bound}) for any $n$ and $k$.

\begin{lemma} \label{lemma:upper_bound_minimum_number_regions}
Let $\mathcal{P}_{n,k}$ be the set of all CPWL functions with exactly $k$ distinct linear components such that $p\colon\mathbb{R}^n\to\mathbb{R},\forall p\in\mathcal{P}_{n,k}$.
Let $\mathcal{C}_{n,k}(p)$ be the collection of all families of closed convex subsets satisfying Definition \ref{def:continuous_piecewise_linear_fn} for any $p\in\mathcal{P}_{n,k}$. Then,
$ \label{eq:k_q_new_bound}
    k\leq\min_{\mathcal{Q}\in\mathcal{C}_{n,k}(p)} \abs{\mathcal{Q}}\leq\phi(n,k).
$
\end{lemma}
The proof of Lemma \ref{lemma:upper_bound_minimum_number_regions} is deferred to Appendix \ref{proof:upper_bound_minimum_number_regions} in the supplementary material. Clearly, $\phi(n,k)$ is a better upper bound of $q$ compared to the bound $q\leq k!$ given by \cite{he2020relu}. When $n$ grows sufficiently slower than $k$, the bound $\phi(n,k)$ can be exponentially smaller than $k!$.

\subsection{Limitations} \label{limitations}
Although these new bounds are significantly better than previous results, it is still possible to find a ReLU network whose hidden neurons are fewer than the bounds in Theorem \ref{theorem:ReLu_network_upper_bound_k_q} to exactly represent a given CPWL function. A tight bound for the case when $n=1$ was first given by Theorem 2.2 in \citep{arora2018understanding}. However, it seems more difficult to bound the size of a network from below for $n>1$. To the best of our knowledge, we are not aware of any tight bounds in the literature for the size of the ReLU network representing a general CPWL function using an arbitrary input dimension.

\section{Representations of CPWL functions have different implications on depth} \label{representations_CPWL_implications}
We reveal implications of using different representations of CPWL functions and their impact on constructing ReLU networks. We first discuss the popular representation used by prior work and the implicit constraint imposed by such a representation. 
\subsection{Constrained depth} \label{constrained_depth}
\cite{arora2018understanding}, \cite{he2020relu}, and \cite{hertrich2021towards} proved the same bound for the number of layers, relying on the following representation of a CPWL function
\begin{equation} \label{eq:convex_pwl_representation}
    p(\mathbf{x})=\sum_{j=1}^J\sigma_j\max_{i\in\eta(j)}f_i(\mathbf{x})
\end{equation}
where $\sigma_j\in\{+1,-1\}$ and $\eta(j)$ is a subset of $[J]$ such that $\abs{\eta(j)}\leq n+1$ for all $j\in[J]$. That is, a sum of a finite number of \textit{max-$\eta$-affine} functions whose signs may be flipped. (\ref{eq:convex_pwl_representation}) was established by Theorem 1 in \citep{wang2005generalization} which is essentially the same as Theorem 1 in \citep{wang2004general} that emphasizes the difference between two convex piecewise linear functions. This result was also used by \cite{goodfellow2013maxout} to prove Proposition 4.1 in the maxout network paper.

The depth given by (\ref{eq:prior_bound_on_depth}) does not scale with the complexity of a CPWL function. This feature directly comes from using a ReLU network to realize each of \textit{max-$\eta$-affine} functions in (\ref{eq:convex_pwl_representation}) and concatenating all of them together. Because the size of $\eta(j)$ is bounded from above by $n+1$, the depth can be made to depend solely on $n$. Such a treatment seems to be the only way if one considers a CPWL function represented by (\ref{eq:convex_pwl_representation}). As a result, the ReLU network is forced to use a depth constrained by the input dimension to represent the given CPWL function, which in turn requires more hidden neurons.
Because we do not use (\ref{eq:convex_pwl_representation}), our networks are not limited by such an implication.

\subsection{Proof sketch for the unconstrained depth} \label{proof_sketches}
We give a proof sketch in this subsection for our main results. By using a different representation, the depth of a ReLU network is able to be scaled with the complexity measure, i.e., the number of pieces, of any given CPWL function to accommodate the high expressivity.

By Theorem 4.2 in \citep{tarela1999region}, any CPWL function $p$ can be represented as
\begin{equation} \label{eq:max_min_representation_main_text}
    p(\mathbf{x}) = \\\max_{\mathcal{X}\in\mathcal{Q}}\min_{i\in\mathcal{A}\left(\mathcal{X}\right)}f_{i}(\mathbf{x})
\end{equation}
for all $\mathbf{x}\in\mathbb{R}^n$ where
$
    \mathcal{A}\left(\mathcal{X}\right)=\left\{i\in[k]\left|\right.f_{i}(\mathbf{x})\geq p(\mathbf{x}),\forall \mathbf{x}\in\mathcal{X}\right\}
$
is the set of indices of linear components that have values greater than or equal to $p(\mathbf{x})$ for all $\mathbf{x}\in\mathcal{X}$, and $\mathcal{Q}$ is any family of closed convex subsets of $\mathbb{R}^n$ satisfying Definition \ref{def:continuous_piecewise_linear_fn}. We have used $f_{1},f_{2},\cdots,f_{k}$ to denote the $k$ distinct linear components of $p$.
Notice that Theorem 4.2 in \citep{tarela1999region} was first stated by Theorem 7 in \citep{tarela1990representation}. Both are essentially the same, but Theorem 4.2 in \citep{tarela1999region} emphasizes the convexity of each of the regions in the domain. Both theorems are also fundamentally equivalent to Theorem 2.1 in \citep{ovchinnikov2002max}. Notice that one of the concluding remarks in \citep{ovchinnikov2002max} pointed out that the convexity of the input space is an essential assumption. The entire space $\mathbb{R}^n$ satisfies such an assumption. In addition, \citeauthor{ovchinnikov2002max} pointed out that the max-min representation also holds for vector-valued CPWL functions. Hence, it is possible to generalize our bounds to vector-valued CPWL functions.

Using the representation in (\ref{eq:max_min_representation_main_text}) and Lemma \ref{lemma:ReLU_network_represent_maximum} below, we are able to prove Theorem \ref{theorem:ReLu_network_upper_bound_k_q} by bounding the size of $\mathcal{Q}$ and $\mathcal{A}\left(\mathcal{X}\right)$. Theorem \ref{theorem:ReLu_network_upper_bound_q} and \ref{theorem:ReLu_network_upper_bound_k} can be proved by applying Lemma \ref{lemma:upper_bound_minimum_number_regions} to Theorem \ref{theorem:ReLu_network_upper_bound_k_q}.
Note that the size $\abs{\mathcal{Q}}$ in (\ref{eq:max_min_representation_main_text}) is the key for the depth of a ReLU network to be able to scale with $q$.
\begin{lemma} \label{lemma:ReLU_network_represent_maximum}
Let $m$ be any positive integer. Define
$
    l(m)=\lceil\log_2m\rceil+1,
$
$
w(m)=\mathbb{I}\left[m>1\right]\left\lceil\frac{3m}{2}\right\rceil,
$
and the following sequence for any positive integer $k$,
\begin{equation} \label{eq:num_neurons_sequence}
    r(k)=\begin{cases}0, &\text{if $k=1$},\\\frac{3k}{2}+r\left(\frac{k}{2}\right), &\text{if $k$ is even},\\2+\frac{3(k-1)}{2}+r\left(\frac{k+1}{2}\right), &\text{if $k\neq 1$ and $k$ is odd}.\end{cases}
\end{equation}
Then, there exists an $l(m)$-layer ReLU network $g\colon\mathbb{R}^n\to\mathbb{R}$ with $r(m)$ hidden neurons and a maximum width of $w(m)$ such that $g$ computes the extremum of $f_1(\mathbf{x}),f_2(\mathbf{x}),\cdots,f_m(\mathbf{x})$, i.e.,
$
    g(\mathbf{x})=\max_{i\in[m]} f_i(\mathbf{x})$ or $g(\mathbf{x})=\min_{i\in[m]} f_i(\mathbf{x})
$
for all $\mathbf{x}\in\mathbb{R}^n$ under any $m$ scalar-valued affine functions $f_1,f_2,\cdots,f_m$. Furthermore, Algorithm \ref{alg:find_a_extremum_ReLU_network} finds such a network in $\mathsf{poly}(m,n)$ time.
\end{lemma}
The proof of Lemma \ref{lemma:ReLU_network_represent_maximum} is deferred to Appendix \ref{proof:ReLU_network_represent_maximum} in the supplementary material. One can also view $l(m)$, $w(m)$, and $r(m)$ as upper bounds for the number of layers, maximum width, and the number of hidden neurons. Because 
$
    r(m)<6m-3
$
by Lemma \ref{lemma:r_upper_bound_formula}, the bound for the number of hidden neurons $r(m)$ is tighter than the bound $8m-4$ given by Lemma D.3 in \citep{arora2018understanding} or Lemma 5.4 in \citep{he2020relu} (these two lemmas are essentially the same). The bound for the number of layers remains the same as the one given by Lemma D.3 in \citep{arora2018understanding}. By combining Lemma \ref{lemma:ReLU_network_represent_maximum} with Lemma \ref{lemma:ReLU_network_represent_identity}, Lemma \ref{lemma:composition_ReLU_networks}, and Lemma \ref{lemma:concat_kReLU_networks}, we can easily perform the same job on computing the extremum of multiple scalar-valued ReLU networks as Lemma D.3 does in \citep{arora2018understanding}. Lemma \ref{lemma:r_upper_bound_formula}, \ref{lemma:ReLU_network_represent_identity}, \ref{lemma:composition_ReLU_networks}, and \ref{lemma:concat_kReLU_networks} are given in Appendix \ref{lemmas} in the supplementary material.
\section{Broader impact}
Our results guarantee that any CPWL function can be exactly computed by a ReLU neural network at a more manageable cost. This assurance is crucial because CPWL functions are important tools in many applications. Such an assurance also relates DNNs closer to CPWL functions and allows researchers and engineers to understand the expressivity of DNNs from a different perspective. We focus on simple ReLU networks (ReLU multi-layer perceptrons) in this paper, but it may be possible to derive bounds for other activation functions and advanced neural network architectures such as maxout networks \citep{goodfellow2013maxout}, residual networks \citep{he2016deep}, densely connected networks \citep{huang2017densely}, and other nonlinear networks \citep{chen2021resnests}, by making some (possibly mild) assumptions. Our contributions advance the fundamental understanding of the link between ReLU networks and CPWL functions.

\begin{ack}
We would like to thank the anonymous reviewers for their constructive comments, Tai-Hsuan Chung for answering our mathematical questions, and Christoph Hertrich for his thoughtful comments on the time complexity of Algorithm \ref{alg:find_a_network_for_a_piecewise_linear_fn} and for clarifying Theorem 4.4 in \citep{hertrich2021towards}. This work was supported in part by NSF under Grant CCF-2225617, Grant CCF-2124929, and Grant IIS-1838897, in part by NIH/NIDCD under Grant R01DC015436, and in part by KIBM Innovative Research Grant Award.
\end{ack}

\bibliography{ref}

\section*{Checklist}

\begin{enumerate}

\item For all authors...
\begin{enumerate}
  \item Do the main claims made in the abstract and introduction accurately reflect the paper's contributions and scope?
    \answerYes{The main claims are summarized in Section \ref{contributions}.}
  \item Did you describe the limitations of your work?
    \answerYes{See Section \ref{limitations}.}
  \item Did you discuss any potential negative societal impacts of your work?
    \answerNo{This is a theoretical paper and we are not aware of any negative societal impacts.}
  \item Have you read the ethics review guidelines and ensured that your paper conforms to them?
    \answerYes{}
\end{enumerate}

\item If you are including theoretical results...
\begin{enumerate}
  \item Did you state the full set of assumptions of all theoretical results?
    \answerYes{All definitions and assumptions are stated or referenced before the results.}
        \item Did you include complete proofs of all theoretical results?
    \answerYes{Appendix \ref{proof} in the supplementary material.}
\end{enumerate}

\item If you ran experiments...
\begin{enumerate}
  \item Did you include the code, data, and instructions needed to reproduce the main experimental results (either in the supplemental material or as a URL)?
    \answerNA{}
  \item Did you specify all the training details (e.g., data splits, hyperparameters, how they were chosen)?
    \answerNA{}
        \item Did you report error bars (e.g., with respect to the random seed after running experiments multiple times)?
    \answerNA{}
        \item Did you include the total amount of compute and the type of resources used (e.g., type of GPUs, internal cluster, or cloud provider)?
    \answerNA{}
\end{enumerate}

\item If you are using existing assets (e.g., code, data, models) or curating/releasing new assets...
\begin{enumerate}
  \item If your work uses existing assets, did you cite the creators?
    \answerNA{}
  \item Did you mention the license of the assets?
    \answerNA{}
  \item Did you include any new assets either in the supplemental material or as a URL?
    \answerNA{}
  \item Did you discuss whether and how consent was obtained from people whose data you're using/curating?
    \answerNA{}
  \item Did you discuss whether the data you are using/curating contains personally identifiable information or offensive content?
    \answerNA{}
\end{enumerate}

\item If you used crowdsourcing or conducted research with human subjects...
\begin{enumerate}
  \item Did you include the full text of instructions given to participants and screenshots, if applicable?
    \answerNA{}
  \item Did you describe any potential participant risks, with links to Institutional Review Board (IRB) approvals, if applicable?
    \answerNA{}
  \item Did you include the estimated hourly wage paid to participants and the total amount spent on participant compensation?
    \answerNA{}
\end{enumerate}

\end{enumerate}


\clearpage
\appendix

\section{Lemmas} \label{lemmas}
\begin{lemma} \label{lemma:ReLU_network_represent_identity}
Let $l$ be any positive integer. There exists an $l$-layer ReLU network $g$ with $2n(l-1)$ hidden neurons and a maximum width of $2n$ such that $g(\mathbf{x})=\mathbf{x}$ for all $\mathbf{x}\in\mathbb{R}^n$. Furthermore, Algorithm \ref{alg:find_an_identity_ReLU_network} finds such a network in $\mathsf{poly}(n,l)$ time.
\end{lemma}
\begin{proof}
Appendix \ref{proof:ReLU_network_represent_identity}.
\end{proof}

\begin{definition}
Let $g_{(l,n,w)}$ denote an $l$-layer ReLU network with $n$ hidden neurons and a maximum width bounded from above by $w$.
\end{definition}

\begin{lemma} \label{lemma:composition_ReLU_networks}
There exists $g_{\left(l_1+l_2-1,n_1+n_2,\max(w_1,w_2)\right)}$ that represents any composition of $g_{(l_1,n_1,w_1)}$ and $g_{(l_2,n_2,w_2)}$. Algorithm \ref{alg:find_a_composited_ReLU_network} finds such a network computing the composition in $\mathsf{poly}\left(\max(w_1,w_2),\max(l_1,l_2)\right)$ time.
\end{lemma}
\begin{proof}
Appendix \ref{proof:composition_ReLU_networks}.
\end{proof}

\begin{lemma}  \label{lemma:strictly_increasing_sequence}
The sequence $r(k)$ defined by (\ref{eq:num_neurons_sequence}) is a strictly increasing sequence.
\end{lemma}
\begin{proof}
Appendix \ref{proof:strictly_increasing_sequence}.
\end{proof}

\begin{lemma} \label{lemma:r_upper_bound_formula}
For any positive integer $k$, the sequence $r(k)$ defined by (\ref{eq:num_neurons_sequence}) satisfies
\begin{equation}
    r(k)\leq 3\left(2^{\left\lceil\log_2k\right\rceil}-1\right)< 6k-3.
\end{equation}
\end{lemma}
\begin{proof}
Appendix \ref{proof:r_upper_bound_formula}
\end{proof}

\begin{lemma} \label{lemma:concat_2ReLU_networks}
Let $m_1$ and $m_2$ be the output dimensions of $g_{(l_1,n_1,w_1)}$ and $g_{(l_2,n_2,w_2)}$, respectively. Define
\begin{equation}
    l=\max(l_1,l_2),
\end{equation}
\begin{equation}
    w=w_j+\max(w_i,2m_i),
\end{equation}
and
\begin{equation}
    n=n_1+n_2+2m_i\abs{l_1-l_2},
\end{equation}
where $i=\arg\min_{k\in[2]}l_k$ and $j=[2]\setminus\{i\}$. Then, there exists $g_{(l,n,w)}$ such that
\begin{equation}
    g_{(l,n,w)}(\mathbf{x})=\begin{bmatrix}g_{(l_1,n_1,w_1)}(\mathbf{x})\\g_{(l_2,n_2,w_2)}(\mathbf{x})\end{bmatrix}
\end{equation}
for all $\mathbf{x}\in\mathbb{R}^n$.
\end{lemma}
\begin{proof}
Appendix \ref{proof:concat_2ReLU_networks}.
\end{proof}

\begin{lemma} \label{lemma:concat_kReLU_networks}
Let $m_1,m_2,\cdots,m_k$ be the output dimensions of $g_{(l_1,n_1,w_1)},g_{(l_2,n_2,w_2)},\cdots,g_{(l_k,n_k,w_k)}$, respectively. Define
\begin{equation}
    l=\max_{i\in[k]}l_i,
\end{equation}
\begin{equation}
    w=\sum_{i\in[k]}\max(w_i,2m_i),
\end{equation}
and
\begin{equation}
    n=\sum_{i\in[k]}n_i+2m_i(l-l_i).
\end{equation}
Then, there exists $g_{(l,n,w)}$ such that
\begin{equation}
    g_{(l,n,w)}(\mathbf{x})=\begin{bmatrix}g_{(l_1,n_1,w_1)}(\mathbf{x})\\g_{(l_2,n_2,w_2)}(\mathbf{x})\\\vdots\\g_{(l_k,n_k,w_k)}(\mathbf{x})\end{bmatrix}
\end{equation}
for all $\mathbf{x}\in\mathbb{R}^n$. Furthermore, Algorithm \ref{alg:find_a_concat_ReLU_network} finds such a network in $\mathsf{poly}\left(\max_{i\in[k]}w_i,k,l\right)$ time.
\end{lemma}
\begin{proof}
Appendix \ref{proof:concat_kReLU_networks}.
\end{proof}

\begin{lemma} \label{lemma:upper_bound_num_feasible_ascending_orders}
Let $f_1,f_2,\cdots,f_k$ be any affine functions such that $f_i\colon\mathbb{R}^n\to\mathbb{R}$ for all $i\in[k]$. Define the set of feasible ascending orders as
\begin{equation}
    \mathcal{S}_{f_1,f_2,\cdots,f_k}^n=\left\{(s_{1},s_{2},\cdots,s_{k})\in\mathfrak{S}(k)\left|\right.f_{s_1}(\mathbf{x})\leq f_{s_2}(\mathbf{x})\leq \cdots\leq  f_{s_{k}}(\mathbf{x}), \mathbf{x}\in\mathbb{R}^n\right\}
\end{equation}
where $\mathfrak{S}(k)$ is the collection of all permutations of the set $[k]$. It holds true that
\begin{equation} \label{eq:bound_for_num_unique_order_regions}
    \abs{\mathcal{S}_{f_1,f_2,\cdots,f_k}^n}\leq \min\left(\sum_{i=0}^n \binom{\frac{k^2-k}{2}}{i},k!\right).
\end{equation}
\end{lemma}
\begin{proof}
Appendix \ref{proof:upper_bound_num_feasible_ascending_orders}.
\end{proof}

\begin{lemma} \label{lemma:existence_intersecting_subsets}
If Definition \ref{def:continuous_piecewise_linear_fn} is satisfied for a non-affine function, then every nonempty subset has a nonempty intersection with the other subset or at least one of the other subsets.
\end{lemma}
\begin{proof}
Appendix \ref{proof:existence_intersecting_subsets}.
\end{proof}

\begin{assumption} \label{assumption:connected_minimum}
The number of closed connected subsets satisfying Definition \ref{def:continuous_piecewise_linear_fn} is a minimum.
\end{assumption}

The interior and frontier (boundary) of a set $\mathcal{X}$ are denoted as $\textnormal{Int}\mathcal{X}$ and $\textnormal{Fr}\mathcal{X}$, respectively.

\begin{lemma} \label{lemma:merge1}
Let $f_i$ denote the affine function associated with $\mathcal{X}_i$ for $i\in[I]$ where $\{\mathcal{X}_i\}_{i\in[I]}$ is a family of closed connected subsets satisfying Assumption \ref{assumption:connected_minimum}. Then, for any $i\in[I],j\in[I]$ such that $i\neq j$ and $\mathcal{X}_i\bigcap\mathcal{X}_j\neq\emptyset$,
\begin{enumerate}[label=(\alph*)]
    \item $f_i$ and $f_j$ are different, and $\{\left.\mathbf{x}\in\mathbb{R}^n\right|f_i(\mathbf{x})=f_j(\mathbf{x})\}\neq\emptyset$. \label{lemma:intersecting_nonempty_intersection}
    \item $\{\left.\mathbf{x}\in\mathbb{R}^n\right|f_i(\mathbf{x})=f_j(\mathbf{x})\}$ is an affine subspace of $\mathbb{R}^n$ with dimension $n-1$. \label{lemma:intersection_affine_subspace}
    \item $\mathcal{X}_i\bigcap\mathcal{X}_j\subseteq\{\left.\mathbf{x}\in\mathbb{R}^n\right|f_i(\mathbf{x})=f_j(\mathbf{x})\}$. \label{lamma:intersection_included_in_affine_subspace}
    \item $\mathbf{x}\not\in\textnormal{Int}\mathcal{X}_i$ and $\mathbf{x}\not\in\textnormal{Int}\mathcal{X}_j$ for all $\mathbf{x}\in\mathcal{X}_i\bigcap\mathcal{X}_j$. \label{lemma:intersection_is_boundary}
\end{enumerate}
\end{lemma}
\begin{proof}
    Appendix \ref{proof:intersecting_nonempty_intersection}, \ref{proof:intersection_affine_subspace}, \ref{proof:intersection_included_in_affine_subspace}, and \ref{proof:intersection_is_boundary}.
\end{proof}
\begin{lemma} \label{lemma:merge2}
If a family of closed connected subsets $\{\mathcal{X}_i\}_{i\in[I]}$ satisfies Assumption \ref{assumption:connected_minimum}, then, for all $i\in[I]$,
\begin{enumerate}[label=(\alph*)]
    \item $\textnormal{Int}\mathcal{X}_i\neq\emptyset$. \label{lemma:nonempty_interior}
    \item $\text{Fr}\mathcal{X}_i=\bigcup_{k\in[I]\setminus i}\mathcal{X}_k\bigcap\mathcal{X}_i$. \label{lemma:boundary_union_other_subsets}
    \item $\textnormal{Int}\mathcal{X}_i\bigcap\textnormal{Int}\mathcal{X}_j=\emptyset$ for all $j\in[I]$ such that $j\neq i$. \label{lemma:nonempty_and_disjoint_interiors}
\end{enumerate}
\end{lemma}
\begin{proof}
Appendix \ref{proof:nonempty_interior}, \ref{proof:boundary_union_other_subsets}, and \ref{proof:nonempty_and_disjoint_interiors}.
\end{proof}

\begin{lemma} \label{lemma:subsets_cannot_be_included_union_affine_subspaces}
Let $\{\mathcal{X}_i\}_{i\in[m]}$ be any finite family of subsets satisfying Assumption \ref{assumption:connected_minimum}. Let $\{\mathcal{H}_{j}\}_{j\in[k]}$ be any finite family of affine subspaces of $\mathbb{R}^n$ with dimension $n-1$. Then, for every $i\in[m]$,
\begin{equation} \label{eq:subset_not_included_union_finite}
\mathcal{X}_i\bigcap \left(\mathbb{R}^n\setminus \bigcup_{j\in[k]} \mathcal{H}_{j}\right)\neq\emptyset.
\end{equation}
\end{lemma}
\begin{proof}
Appendix \ref{proof:subsets_cannot_be_included_union_affine_subspaces}.
\end{proof}

\begin{proposition} \label{proposition:largest_closed_connected_minimum}
For any family of closed connected subsets satisfying Definition \ref{def:continuous_piecewise_linear_fn}, all subsets are the largest closed connected subsets if and only if Assumption \ref{assumption:connected_minimum} is satisfied.
\end{proposition}
\begin{proof}
Appendix \ref{proof:largest_closed_connected_minimum}.
\end{proof}
\section{Proofs} \label{proof}
\subsection{Proof of Lemma \ref{lemma:upper_bound_minimum_number_regions}} \label{proof:upper_bound_minimum_number_regions}
\begin{proof}
Let the family of closed connected subsets $\bar{\mathcal{Q}}=\{\mathcal{X}_i\}_{i\in[I]}$ satisfy Assumption \ref{assumption:connected_minimum} for any $p\in\mathcal{P}_{n,k}$.
Let the $k$ distinct linear components of $p$ be $f_1,f_2,\cdots,f_k$ and $\mathcal{H}_{lm}$ be the intersection between $f_l$ and $f_m$ for $l\in[k],m\in[k],l\neq m$. Note that every $\mathcal{H}_{lm}$ is an affine subspace of $\mathbb{R}^n$ with dimension $n-1$ (a hyperplane) or an empty set.
Because the linear components are distinct, it must be true that $k\leq I$ by Definition \ref{def:linear_components}. If $p$ is an affine function, then it follows that $k=\min_{\bar{\mathcal{Q}}\in\mathcal{C}_{n,k}(p)} \abs{\bar{\mathcal{Q}}}=\phi(n,k)=1$, the claim holds. For the non-affine case, we must have $k>1$.

Let $\mathcal{R}=\mathbb{R}^n\setminus\mathcal{H}$ where
\begin{equation}
    \mathcal{H}=\bigcup_{k\in[m],l\in[m],k\neq l}\mathcal{H}_{kl}.
\end{equation}
Note that $\mathcal{H}\neq\emptyset$ according to Lemma \ref{lemma:existence_intersecting_subsets} and \ref{lemma:merge1}\ref{lamma:intersection_included_in_affine_subspace}. By Lemma \ref{lemma:merge2}\ref{lemma:boundary_union_other_subsets}, the boundary or frontier of $\mathcal{X}_i$ for $i\in[I]$ is given by
\begin{equation}
    \text{Fr} \mathcal{X}_i=\bigcup_{j\in[I]\setminus i}\left(\mathcal{X}_i\bigcap\mathcal{X}_j\right).
\end{equation}
Because every $\mathcal{X}_i\bigcap\mathcal{X}_j$ for $i\in[I],j\in[I],i\neq j$ is a subset of some $\mathcal{H}_{lm}$ for $l\in[k],m\in[k],l\neq m$ by Lemma \ref{lemma:merge1}\ref{lamma:intersection_included_in_affine_subspace}, it follows that the boundary of $\mathcal{X}_i$, $\text{Fr} \mathcal{X}_i$, satisfies
\begin{equation} \label{eq:boundary_condition}
    \text{Fr} \mathcal{X}_i\subseteq\mathcal{H}
\end{equation}
for $i\in[I]$. The interior of $\mathcal{X}_i$, $\text{Int} \mathcal{X}_i$, is a nonempty subset of $\mathbb{R}^n$ according to Lemma \ref{lemma:merge2}\ref{lemma:nonempty_interior}. Furthermore, by Lemma \ref{lemma:subsets_cannot_be_included_union_affine_subspaces},
\begin{equation} \label{eq:interior_condition}
    \mathcal{X}_i\bigcap\mathcal{R}\neq\emptyset.
\end{equation}
Now, define
\begin{equation}
    \mathcal{Z}_i=\left(\text{Int} \mathcal{X}_i\right)\bigcap\mathcal{R}
\end{equation}
for $i\in[I]$. Note that $\mathcal{Z}_i=\mathcal{X}_i\bigcap\mathcal{R}\neq\emptyset$ due to (\ref{eq:boundary_condition}) and (\ref{eq:interior_condition}).
Let $\mathcal{A}$ be any subset of $\mathbb{R}^n$ and $\lambda(\mathcal{A})$ be the number of connected components of $\mathcal{A}$ in $\mathbb{R}^n$.
It must be true that
\begin{equation}
    1=\lambda(\mathcal{X}_i)\leq\lambda\left(\text{Int} \mathcal{X}_i\right)\leq\lambda(\mathcal{Z}_i).
\end{equation}
By Lemma \ref{lemma:merge2}\ref{lemma:nonempty_and_disjoint_interiors}, $\text{Int} \mathcal{X}_i\bigcap\text{Int} \mathcal{X}_j=\emptyset$ for $i\in[I],j\in[I],i\neq j$. We have
\begin{equation} \label{eq:bounding_I_Z}
    I\leq\lambda\left(\bigcup_{i\in[I]}\text{Int}\mathcal{X}_i\right)=\sum_{i\in[I]}\lambda(\text{Int}\mathcal{X}_i)\leq\sum_{i\in[I]}\lambda(\mathcal{Z}_i)=\lambda\left(\bigcup_{i\in[I]}\mathcal{Z}_i\right)=\lambda\left(\bigcup_{i\in[I]}\mathcal{X}_i\bigcap\mathcal{R}\right).
\end{equation}
Notice that
\begin{equation} \label{eq:X_i_P}
    \bigcup_{i\in[I]}\mathcal{X}_i\bigcap\mathcal{R}=\mathbb{R}^n\bigcap\mathcal{R}=\mathcal{R}
\end{equation}
by the property $\bigcup_{i\in[I]}\mathcal{X}_i=\mathbb{R}^n$ in Definition \ref{def:continuous_piecewise_linear_fn}. Plugging (\ref{eq:X_i_P}) into (\ref{eq:bounding_I_Z}) leads to
\begin{equation}
    I\leq\lambda(\mathcal{R})
\end{equation}
which states that $I$ is bounded from above by the number of connected components of $\mathcal{R}$ in $\mathbb{R}^n$. Notice that every component is an open convex set because every component is the intersection of a finite number of open half spaces.
Therefore,
\begin{equation}
    I=\abs{\bar{\mathcal{Q}}}=\min_{\mathcal{Q}'\in\mathcal{C}_{n,k}'(p)} \abs{\mathcal{Q}'}\leq\min_{\mathcal{Q}\in\mathcal{C}_{n,k}(p)} \abs{\mathcal{Q}}\leq\lambda(\mathcal{R})
\end{equation}
where $\mathcal{C}_{n,k}'(p)$ denotes the collection of all families of closed connected subsets satisfying Definition \ref{def:continuous_piecewise_linear_fn} for any $p\in\mathcal{P}_{n,k}$.
Because the ascending order of these $k$ linear components does not change within a connected component of $\mathcal{R}$, $\lambda(\mathcal{R})$ can be bounded from above by the number of feasible ascending orders. Let $\mathfrak{S}(k)$ be the collection of all permutations of the set $[k]$. It follows that
\begin{equation} \label{eq:lambda_bound_feasible_ascending_orders}
    \lambda(\mathcal{R})\leq\abs{\left\{(s_{1},s_{2},\cdots,s_{k})\in\mathfrak{S}(k)\left|\right.f_{s_1}(\mathbf{x})\leq f_{s_2}(\mathbf{x})\leq \cdots\leq  f_{s_{k}}(\mathbf{x}), \mathbf{x}\in\mathbb{R}^n\right\}}.
\end{equation}
Finally, Lemma \ref{lemma:upper_bound_num_feasible_ascending_orders} proves the statement by bounding the number of feasible ascending orders.
\end{proof}
\subsection{Proof of Lemma \ref{lemma:ReLU_network_represent_maximum}} \label{proof:ReLU_network_represent_maximum}
\begin{proof}
It suffices to show that
\begin{equation}
    g(\mathbf{x})=\max_{i\in[k]} x_i.
\end{equation}
for all $\mathbf{x}=\begin{bmatrix}x_1&x_2&\cdots&x_k\end{bmatrix}^\mathsf{T}\in\mathbb{R}^k$ since the composition of affine functions is still affine. The affine functions can be absorbed into the first layer of the ReLU network $g$.
We prove the case for taking the maximum of $m$ real numbers since the same procedure below can be applied to prove the case of taking the minimum due to the following identity
\begin{equation}
    \min_{i\in[k]}f_i(\mathbf{x})=-\max_{i\in[k]}-f_i(\mathbf{x}).
\end{equation}
Because $\max(x_1,x_2)=\max(0,x_2-x_1)+\max(0,x_1)-\max(0,-x_1)$ for any $x_1\in\mathbb{R}$ and $x_2\in\mathbb{R}$, it holds true that
\begin{equation} \label{eq:maximum_over_K_real_numbers}
        \max_{j\in[k]}x_j=\begin{cases}\max_{j\in\left[\frac{k}{2}\right]}\max_{i\in\{2j-1,2j\}}x_i, &\text{if $k$ is even}\\\max_{j\in\left[\frac{k+1}{2}\right]}\alpha(j;x_1,x_2,\cdots,x_k), &\text{if $k$ is odd}\end{cases}
\end{equation}
for $x_j\in\mathbb{R},j\in[k]$ where
\begin{equation}
    \alpha(j;x_1,x_2,\cdots,x_k)=\begin{cases}\max_{i\in\{2j-1,2j\}}x_i, &\text{if }  j\in\left[\frac{k-1}{2}\right]\\\max(0,x_k)-\max(0,-x_k), &\text{if }j=\frac{k+1}{2}\end{cases}.
\end{equation}
Let $r(k)$ be the number of operations of taking the maximum between a zero and a real number, i.e., $\max(0,x),x\in\mathbb{R}$ for computing the maximum of $k$ real numbers using (\ref{eq:maximum_over_K_real_numbers}). One can find $r(2)=3$ and $r(3)=8$ by expanding all operations in (\ref{eq:maximum_over_K_real_numbers}). Because we do not need any maximum operations to compute the maximum over a singleton, we define $r(1)=0$. For any positive integer $k$ such that $k\geq 2$, we have the recursion
\begin{equation} \label{eq:recursion}
    r(k)=\begin{cases}\frac{3k}{2}+r\left(\frac{k}{2}\right), &\text{if $k$ is even}\\2+\frac{3(k-1)}{2}+r\left(\frac{k+1}{2}\right), &\text{if $k$ is odd}\end{cases}
\end{equation}
according to (\ref{eq:maximum_over_K_real_numbers}). Note that $r(n)$ is the number of ReLUs in a ReLU network $g$ that computes the maximum of $n$ real numbers or a \textit{max-affine} function. The number of ReLUs here is equivalent to the number of hidden neurons according to Definition \ref{def:plain_ReLU_network}.
We shall note that the number of ReLU layers is equivalent to the number of hidden layers.

Obviously, we only need a $1$-layer ReLU network with no ReLUs to compute the maximum of a singleton.
Suppose that we aim to compute the maximum of $m=2^n$ real numbers for any positive integer $n$. Then, every time the recursion goes to the next level in (\ref{eq:recursion}), the number of variables considered for computing the maximum is halved. Hence, the number of ReLU layers is $n$. When $m$ is not a power of two, i.e., $2^{n}<m<2^{n+1}$, then we can always construct a ReLU network with $n+1$ ReLU layers and $2^{n+1}$ input neurons, and set weights connected to the $2^{n+1}-m$ ``phantom input neurons'' to zeros. Because $\lceil\log_2m\rceil=n+1$ for $2^{n}<m<2^{n+1}$, the number of ReLU layers is $\lceil\log_2m\rceil$ for any positive integer $m$. By Definition \ref{def:num_hidden_neurons_width}, we have $l(m)=\lceil\log_2m\rceil+1$.

By Lemma \ref{lemma:strictly_increasing_sequence}, $r(k)$ is a strictly increasing sequence. Therefore, the maximum width of the network is given by the width of the first hidden layer. When $L=1$ or $m=1$, the width is $0$ due to Definition \ref{def:num_hidden_neurons_width}. When $L>1$ or $m>1$,
\begin{equation}
    \begin{split}
        \max_{l\in[L-1]}k_l &= \begin{cases}\frac{3m}{2}, &\text{if $m$ is even}\\2+\frac{3(m-1)}{2}, &\text{if $m$ is odd}\end{cases}\\
        &=\left\lceil\frac{3m}{2}\right\rceil.
    \end{split}
\end{equation}
Algorithm \ref{alg:find_a_extremum_ReLU_network} directly follows from the above construction. Its complexity analysis is deferred to Table \ref{tab:find_a_extremum_ReLU_network} in Appendix \ref{algorithm}.
\end{proof}

\subsection{Proof of Theorem \ref{theorem:ReLu_network_upper_bound_k_q}} \label{proof:ReLu_network_upper_bound_k_q}
\begin{proof}

Let $f_{1},f_{2},\cdots,f_{k}$ be $k$ distinct linear components of $p$ and $\mathcal{Q}$ be any family of closed convex subsets of $\mathbb{R}^n$ satisfying Definition \ref{def:continuous_piecewise_linear_fn}.
By Theorem 4.2 in \citep{tarela1999region}, $p$ can be represented as
\begin{equation} \label{eq:max_min_representation}
    p(\mathbf{x}) = \max_{\mathcal{X}\in\mathcal{Q}}\min_{i\in\mathcal{A}\left(\mathcal{X}\right)}f_{i}(\mathbf{x})
\end{equation}
for all $\mathbf{x}\in\mathbb{R}^n$ where
\begin{equation}
    \mathcal{A}\left(\mathcal{X}\right)=\left\{i\in[k]\left|\right.f_{i}(\mathbf{x})\geq p(\mathbf{x}),\forall \mathbf{x}\in\mathcal{X}\right\}
\end{equation}
is the set of indices of linear components that have values greater than or equal to $p(\mathbf{x})$ for all $\mathbf{x}\in\mathcal{X}$. A thorough discussion of the representation (\ref{eq:max_min_representation}) is given in Section \ref{proof_sketches}.

According to (\ref{eq:max_min_representation}), there are $\abs{\mathcal{Q}}$ minima required to be computed where each of them is a minimum of $\abs{\mathcal{A}\left(\mathcal{X}\right)}$ real numbers. Then, the value of $p$ can be computed by taking the maximum of the resulting $\abs{\mathcal{Q}}$ minima. We will show that these operations are realizable by a ReLU network.
By Lemma \ref{lemma:ReLU_network_represent_maximum}, an $l(m)$-layer ReLU network with $r(m)$ hidden neurons and a maximum width of $w(m)$ can compute the extremum of $m$ real numbers given by $m$ affine functions.

We realize (\ref{eq:max_min_representation}) in three steps. First, we create $\abs{\mathcal{Q}}$ ReLU networks where each of them is an $l\left(\abs{\mathcal{A}\left(\mathcal{X}\right)}\right)$-layer ReLU network with $r\left(\abs{\mathcal{A}\left(\mathcal{X}\right)}\right)$ hidden neurons and a maximum width of $w\left(\abs{\mathcal{A}\left(\mathcal{X}\right)}\right)$ that computes $\min_{i\in\mathcal{A}\left(\mathcal{X}\right)}f_{i}(\mathbf{x})$ for $\mathcal{X}\in\mathcal{Q}$. Second, we parallelly concatenate these $\abs{\mathcal{Q}}$ networks, i.e., put them in parallel and let them share the same input to obtain a ReLU network that takes $\mathbf{x}$ and outputs $\abs{\mathcal{Q}}$ real numbers. Finally, we create an $l\left(\abs{\mathcal{Q}}\right)$-layer ReLU network with $r\left(\abs{\mathcal{Q}}\right)$ hidden neurons and a maximum width of $w\left(\abs{\mathcal{Q}}\right)$ that takes the maximum of $\abs{\mathcal{Q}}$ real numbers.

The parallel combination of $\abs{\mathcal{Q}}$ networks in the second step can be realized by Lemma \ref{lemma:concat_kReLU_networks}. The third step can be fulfilled by Lemma \ref{lemma:composition_ReLU_networks}. With the above construction, we can now count the number of layers, the upper bound for the maximum width, and the number of hidden neurons for a ReLU network that realizes $p$.
The number of layers is given by
\begin{equation} \label{eq:exact_num_layers}
    l\left(\abs{\mathcal{Q}}\right)+\max_{\mathcal{X}\in\mathcal{Q}}l\left(\abs{\mathcal{A}\left(\mathcal{X}\right)}\right)-1.
\end{equation} 
The maximum width is bounded from above by
\begin{equation} \label{eq:at_most_maximum_width}
    \max\left(\sum_{\mathcal{X}\in\mathcal{Q}}\max\left(w\left(\abs{\mathcal{A}\left(\mathcal{X}\right)}\right),2\right),w\left(\abs{\mathcal{Q}}\right)\right).
\end{equation}
The number of hidden neurons is given by
\begin{equation} \label{eq:num_hidden_neurons}
    r\left(\abs{\mathcal{Q}}\right)+\sum_{\mathcal{X}\in\mathcal{Q}}r\left(\abs{\mathcal{A}\left(\mathcal{X}\right)}\right)+2\left(\max_{\mathcal{Y}\in\mathcal{Q}}l\left(\abs{\mathcal{A}\left(\mathcal{Y}\right)}\right)-l\left(\abs{\mathcal{A}\left(\mathcal{X}\right)}\right)\right).
\end{equation}

Because $\mathcal{A}\left(\mathcal{X}\right)$ for every $\mathcal{X}\in\mathcal{Q}$ is a subset of $[k]$, it holds that
\begin{equation} \label{eq:upper_bound_A_m}
    1\leq\abs{\mathcal{A}\left(\mathcal{X}\right)}\leq k
\end{equation}
for all $\mathcal{X}\in\mathcal{Q}$. Therefore, the number of layers in (\ref{eq:exact_num_layers}) can be bounded from above by
\begin{equation} \label{eq:exact_num_layers_simplified}
    l\left(\abs{\mathcal{Q}}\right)+l\left(k\right)-1=\left\lceil\log_2\abs{\mathcal{Q}}\right\rceil+\left\lceil\log_2k\right\rceil+1
\end{equation}
where we have used the definition of the function $l$ in Lemma \ref{lemma:ReLU_network_represent_maximum}. Again, using (\ref{eq:upper_bound_A_m}),
the upper bound for the maximum width in (\ref{eq:at_most_maximum_width}) can be further bounded from above by
\begin{equation} \label{eq:at_most_maximum_width_upper_bound}
    \begin{split}
        \max\left(\sum_{\mathcal{X}\in\mathcal{Q}}\max\left(w\left(k\right),2\right),w\left(\abs{\mathcal{Q}}\right)\right)&=\max\left(\abs{\mathcal{Q}}\max\left(w\left(k\right),2\right),w\left(\abs{\mathcal{Q}}\right)\right)\\
        &\leq\max\left(\abs{\mathcal{Q}}\max\left(\left\lceil\frac{3k}{2}\right\rceil,2\right),\left\lceil\frac{3\abs{\mathcal{Q}}}{2}\right\rceil\right)\\
        &=\left\lceil\frac{3k}{2}\right\rceil\abs{\mathcal{Q}}
    \end{split}
\end{equation}
where we have used the definition of the function $w$ in Lemma \ref{lemma:ReLU_network_represent_maximum}. Note that the maximum width is zero when the number of layers is one. Finally, again, using (\ref{eq:upper_bound_A_m}), the number of neurons in (\ref{eq:num_hidden_neurons}) can be bounded from above by
\begin{equation}  \label{eq:num_hidden_neurons_upper_bound}
    \begin{split}
    &r\left(\abs{\mathcal{Q}}\right)-2l(k)+2l(1)+\sum_{\mathcal{X}\in\mathcal{Q}}\left(r\left(k\right)+2l\left(k\right)-2l(1)\right)\\
    &=r\left(\abs{\mathcal{Q}}\right)-2l(k)+2l(1)+\abs{\mathcal{Q}}\left(r\left(k\right)+2l\left(k\right)-2l(1)\right)\\
    &=r\left(\abs{\mathcal{Q}}\right)-2\left\lceil\log_2k\right\rceil+\abs{\mathcal{Q}}\left(r\left(k\right)+2\left\lceil\log_2k\right\rceil\right)\\
    &\leq 3\left(2^{\left\lceil\log_2\abs{\mathcal{Q}}\right\rceil}-1\right)-2\left\lceil\log_2k\right\rceil+\abs{\mathcal{Q}}\left(3\left(2^{\left\lceil\log_2k\right\rceil}-1\right)+2\left\lceil\log_2k\right\rceil\right)\\
    &=3\left(2^{\left\lceil\log_2\abs{\mathcal{Q}}\right\rceil}-1\right)+3\abs{\mathcal{Q}}\left(2^{\left\lceil\log_2k\right\rceil}-1\right)+2\left(\abs{\mathcal{Q}}-1\right)\left\lceil\log_2k\right\rceil
    \end{split}
\end{equation}
where we have used Lemma \ref{lemma:r_upper_bound_formula} for the upper bound in the fourth line of (\ref{eq:num_hidden_neurons_upper_bound}). Expanding and rearranging terms in (\ref{eq:num_hidden_neurons_upper_bound}) lead to (\ref{eq:thm_upper_bound_hidden_neurons}).

Algorithm \ref{alg:find_a_network_for_a_piecewise_linear_fn} directly follows from the above construction. Its complexity analysis is deferred to Table \ref{tab:find_a_network_for_a_piecewise_linear_fn} in Appendix \ref{algorithm}.
\end{proof}

\subsection{Proof of Theorem \ref{theorem:ReLu_network_upper_bound_q}} \label{proof:ReLu_network_upper_bound_q}
\begin{proof}
By Lemma \ref{lemma:upper_bound_minimum_number_regions}, the number of distinct linear components $k$ is bounded from above by the number of pieces, i.e., $k\leq q$, implying that the bounds in Theorem \ref{theorem:ReLu_network_upper_bound_k_q} can be written in terms of $q$.
Substituting $k$ with $q$ in Theorem \ref{theorem:ReLu_network_upper_bound_k_q} proves the claim.

According to Theorem \ref{theorem:ReLu_network_upper_bound_k_q}, the time complexity of Algorithm \ref{alg:find_a_network_for_a_piecewise_linear_fn} is $\mathsf{poly}(n,k,q,L)$. Using the bound $k\leq q$ proves the claim for the time complexity.
\end{proof}

\subsection{Proof of Theorem \ref{theorem:ReLu_network_upper_bound_k}} \label{proof:ReLu_network_upper_bound_k}
\begin{proof}
By Lemma \ref{lemma:upper_bound_minimum_number_regions}, the minimum number of closed convex subsets $q$ of a CPWL function $p\colon\mathbb{R}^n\to\mathbb{R}$ can be bounded from above by $\phi(n,k)$, i.e.,
\begin{equation}
    q\leq\phi(n,k)=\min\left(\sum_{i=0}^n \binom{\frac{k^2-k}{2}}{i},k!\right).
\end{equation}
Substituting $q$ with $\phi(n,k)$ in Theorem \ref{theorem:ReLu_network_upper_bound_k_q} proves the claim.
\end{proof}

\subsection{Proof of Lemma \ref{lemma:ReLU_network_represent_identity}} \label{proof:ReLU_network_represent_identity}
\begin{proof}
Obviously, a one-layer ReLU network is an affine function whose weights can be set to fulfill the identity mapping in $\mathbb{R}^n$. We prove the case when the number of layers is more than one in the next paragraph. We start with a scalar case, and then work on the vector case.

For any $x\in\mathbb{R}$, it holds that $\max(0,x)-\max(0,-x)=x$. In other words, a hidden layer of two ReLUs with $+1$ and $-1$ weights can represent an identity mapping for any scalar. For any vector input in $\mathbb{R}^n$, we can concatenate such structures of two ReLUs in parallel because the identity mapping can be decomposed into $n$ individual identity mappings from $n$ coordinates. Therefore, a two-layer ReLU network with $2n$ hidden neurons can realize the identity mapping in $\mathbb{R}^n$. Stacking such a hidden layer any number of times gives a deeper network that is still an identity mapping.
Algorithm \ref{alg:find_an_identity_ReLU_network} follows from the above construction. Its complexity analysis is deferred to Table \ref{tab:find_an_identity_ReLU_network} in Appendix \ref{algorithm}.
\end{proof}

\subsection{Proof of Lemma \ref{lemma:composition_ReLU_networks}} \label{proof:composition_ReLU_networks}
\begin{proof}
Because a composition of two affine mappings is still affine, the first layer of either one of the two networks can be absorbed into the last layer of the other one if their dimensions are compatible. The resulting new network still satisfies Definition \ref{def:plain_ReLU_network}. The number of layers of the new network is $l_1+l_2-1$. The number of hidden neurons of the new network is $n_1+n_2$. The maximum width of the new network is at most $\max(w_1,w_2)$.
Algorithm \ref{alg:find_a_composited_ReLU_network} follows from the above construction. Its complexity analysis is deferred to Table \ref{tab:find_a_composited_ReLU_network} in Appendix \ref{algorithm}.
\end{proof}

\subsection{Proof of Lemma \ref{lemma:strictly_increasing_sequence}} \label{proof:strictly_increasing_sequence}
\begin{proof}
For any positive even integer $k\geq 4$, it holds true that
\begin{equation} \label{eq:increasing_seq_even}
    \begin{split}
        r(k)-r(k-1)&=\frac{3k}{2}+r\left(\frac{k}{2}\right)-2-\frac{3(k-2)}{2}-r\left(\frac{k}{2}\right)=1.
    \end{split}
\end{equation}
For any positive odd integer $k$ such that $k\geq 3$, we have
\begin{equation} \label{eq:increasing_seq_odd}
    \begin{split}
        r(k)-r(k-1)&=2+\frac{3(k-1)}{2}+r\left(\frac{k+1}{2}\right)-\frac{3(k-1)}{2}-r\left(\frac{k-1}{2}\right)\\
        &=\begin{cases}3, &\text{if $\frac{k+1}{2}$ is even}\\
        2+r\left(\frac{k+1}{2}\right)-r\left(\frac{k+1}{2}-1\right), &\text{otherwise}\end{cases}
    \end{split}
\end{equation}
which is strictly greater than zero. Note that (\ref{eq:increasing_seq_odd}) is greater than $0$ because the equality in (\ref{eq:increasing_seq_odd}) can be applied over and over again to reach (\ref{eq:increasing_seq_even}) or the base case $r(2)-r(1)=3$.
\end{proof}

\subsection{Proof of Lemma \ref{lemma:r_upper_bound_formula}} \label{proof:r_upper_bound_formula}
\begin{proof}
By Lemma \ref{lemma:strictly_increasing_sequence}, $r(k)$ is a strictly increasing sequence. Then, it must be true that
\begin{equation}
    r(k)=r\left(2^{\log_2k}\right)\leq r\left(2^{\left\lceil\log_2k\right\rceil}\right).
\end{equation}
According to the recursion (\ref{eq:num_neurons_sequence}), it holds that
\begin{equation}
    \begin{split}
        r\left(2^{\left\lceil\log_2k\right\rceil}\right)&=\frac{3}{2}\sum_{i=1}^{\left\lceil\log_2k\right\rceil}2^i\\
        &=\frac{3}{2}\left(2^{\left\lceil\log_2k\right\rceil+1}-2\right)\\
        &=3\left(2^{\left\lceil\log_2k\right\rceil}-1\right)\\
        &< 3\left(2^{\left(\log_2k\right)+1}-1\right)\\
        &=3\left(2k-1\right).
    \end{split}
\end{equation}
\end{proof}

\subsection{Proof of Lemma \ref{lemma:concat_2ReLU_networks}} \label{proof:concat_2ReLU_networks}
\begin{proof}
Two ReLU networks can be combined in parallel such that the new network shares the same input and the two output vectors from the two ReLU networks are concatenated together. To see this, we show that the weights of the new network can be found by the following operations.
Let $\mathbf{W}^1_i$ and $\mathbf{b}^1_i$ be the weights of the $i$-th layer in $g_{(l_1,n_1,w_1)}$, and $\mathbf{W}^2_i$ and $\mathbf{b}^2_i$ are the weights of the $i$-th layer in $g_{(l_2,n_2,w_2)}$. Let $\mathbf{W}_i$ and $\mathbf{b}_i$ be the weights of the new network. Now, we find the weights for the new network.
In the first layer, we construct
\begin{equation}
    \mathbf{W}_1=\begin{bmatrix}\mathbf{W}^1_1\\\mathbf{W}^2_1\end{bmatrix}
\end{equation}
and
\begin{equation}
    \mathbf{b}_1=\begin{bmatrix}\mathbf{b}^1_1\\\mathbf{b}^2_1\end{bmatrix}.
\end{equation}
For the $i$-th layer such that $1<i\leq\min(l_1,l_2)$, we use
\begin{equation} \label{eq:W_new_network}
    \mathbf{W}_i=\begin{bmatrix}\mathbf{W}^1_i&\mathbf{0}\\\mathbf{0}&\mathbf{W}^2_i\end{bmatrix}
\end{equation}
and
\begin{equation} \label{eq:b_new_network}
    \mathbf{b}_i=\begin{bmatrix}\mathbf{b}^1_i\\\mathbf{b}^2_i\end{bmatrix}.
\end{equation}
If $l_1=l_2$, then the claim is proved. If $l_1\neq l_2$, then we stack a network that implements the identity mapping to the shallower network such that the numbers of layers of the two networks are the same. Because the network $g_{(l_i,n_i,w_i)}$ is shallower than the other network, we append $\abs{l_1-l_2}$ hidden layers to $g_{(l_i,n_i,w_i)}$ such that the procedure in (\ref{eq:W_new_network}) and (\ref{eq:b_new_network}) can be used.
By Lemma \ref{lemma:ReLU_network_represent_identity}, there exists an $\left(\abs{l_1-l_2}+1\right)$-layer ReLU network $g_{(\abs{l_1-l_2}+1,2m_i\abs{l_1-l_2},2m_i)}$ with $2m_i\abs{l_1-l_2}$ hidden neurons and a maximum width bounded from above by $2m_i$ for representing the identity mapping in $\mathbb{R}^{m_i}$. By Lemma \ref{lemma:composition_ReLU_networks}, there exists a network $g_{\left(l_i+\abs{l_1-l_2},n_i+2m_i\abs{l_1-l_2},\max(w_i,2m_i)\right)}$ that represents the composition of $g_{(\abs{l_1-l_2}+1,2m_i\abs{l_1-l_2},2m_i)}$ and $g_{(l_i,n_i,w_i)}$. Now, (\ref{eq:W_new_network}) and (\ref{eq:b_new_network}) can be used to combine $g_{(l_j,n_j,w_j)}$ and $g_{\left(l_i+\abs{l_1-l_2},n_i+2m_i\abs{l_1-l_2},\max(w_i,2m_i)\right)}$ in parallel because the number of layers in network $g_{\left(l_i+\abs{l_1-l_2},n_i+2m_i\abs{l_1-l_2},\max(w_i,2m_i)\right)}$ is equal to $l_j$ according to the fact that $l_i+\abs{l_1-l_2}=\max(l_1,l_2)=l_j$. Such a new network has $\max(l_1,l_2)$ layers and
\begin{equation}
    n_j+n_i+2m_i\abs{l_1-l_2}=n_1+n_2+2m_i\abs{l_1-l_2}
\end{equation}
hidden neurons. The maximum width of the new network is at most
$
    w_j+\max(w_i,2m_i)
$.
\end{proof}

\subsection{Proof of Lemma \ref{lemma:concat_kReLU_networks}} \label{proof:concat_kReLU_networks}
\begin{proof}
The case $k=1$ is trivial. The case $k=2$ is proved by Lemma \ref{lemma:concat_2ReLU_networks}, which gives a tighter bound on the maximum width. The number of layers and hidden neurons of the claim agree with Lemma \ref{lemma:concat_2ReLU_networks} when $k=2$. The claim can be proved by following a similar procedure from the proof of Lemma \ref{lemma:concat_2ReLU_networks}. By Lemma \ref{lemma:ReLU_network_represent_identity}, we can stack an identity mapping realized by an $(l-l_i+1)$-layer ReLU network with $2m_i(l-l_i)$ hidden neurons and a maximum width of $2m_i$ on the $i$-th network for all $i\in[k]$ such that $l_i<l$. In other words, we increase the number of hidden layers for any network whose number of layers is less than $l$ such that the cascade of the network and the corresponding identity mapping has $l$ layers. For all $i\in[k]$ such that $l_i<l$, the extended network has $n_i+2m_i(l-l_i)$ hidden neurons and a maximum width at most $\max(w_i,2m_i)$ according to Lemma \ref{lemma:composition_ReLU_networks}. Because all the networks now have the same number of layers, we can directly combine them in parallel. Hence, the resulting new network has $\max_{i\in[k]}l_i$ layers and
\begin{equation}
    \sum_{i\in[k]}n_i+2m_i(l-l_i)
\end{equation}
hidden neurons and a maximum width at most
\begin{equation}
    \sum_{i\in[k]}\max(w_i,2m_i).
\end{equation}
Algorithm \ref{alg:find_a_concat_ReLU_network} directly follows from the above construction. Its complexity analysis is deferred to Table \ref{tab:find_a_concat_ReLU_network} in Appendix \ref{algorithm}.
\end{proof}

\subsection{Proof of Lemma \ref{lemma:upper_bound_num_feasible_ascending_orders}} \label{proof:upper_bound_num_feasible_ascending_orders}
\begin{proof}
Because $\mathcal{S}_{f_1,f_2,\cdots,f_k}^n$ is a subset of $\mathfrak{S}(k)$ and $\abs{\mathfrak{S}(k)}=k!$ is the number of permutations of $k$ distinct objects, it follows that
\begin{equation} \label{eq:size_S_k_fac}
    \abs{\mathcal{S}_{f_1,f_2,\cdots,f_k}^n}\leq k!.
\end{equation}
On the other hand, the number of hyperplanes, or affine subspaces of $\mathbb{R}^n$ with dimension $n-1$, induced by the distinct intersections between any two different affine functions is bounded from above by
\begin{equation}
    \binom{k}{2}.
\end{equation}
Let the arrangement of these hyperplanes be $\mathcal{A}$, and $\abs{\mathcal{A}}$ be the number of hyperplanes in the arrangement. By Zaslavsky's Theorem \citep{zaslavsky1975facing}, the number of connected components of the set
\begin{equation}
    \mathbb{R}^n \setminus \bigcup_{H\in\mathcal{A}}H
\end{equation}
is bounded from above by
\begin{equation}
    \sum_{i=0}^n\binom{\abs{\mathcal{A}}}{i}
\end{equation}
Because there are at most $\binom{k}{2}$ hyperplanes in $\mathbb{R}^n$, it follows that
\begin{equation} \label{eq:eq:size_S_sum_com}
    \abs{\mathcal{S}_{f_1,f_2,\cdots,f_k}^n}\leq\sum_{i=0}^n \binom{\binom{k}{2}}{i}.
\end{equation}
Combining (\ref{eq:size_S_k_fac}) and (\ref{eq:eq:size_S_sum_com}) proves the claim. Notice that the ascending order does not change within a connected component.
\end{proof}

\subsection{Proof of Lemma \ref{lemma:existence_intersecting_subsets}} \label{proof:existence_intersecting_subsets}
\begin{proof}
Let $\mathcal{X}_1,\mathcal{X}_2,\cdots,\mathcal{X}_I$ be a family of nonempty subsets satisfying Definition \ref{def:continuous_piecewise_linear_fn} for a non-affine function.
We prove the claim by contradiction. Suppose that there exists at least one nonempty closed subset, say $\mathcal{X}_i$, that is disjoint with every other closed subset $\mathcal{X}_j,j\in[I]\setminus i$. It follows that
\begin{equation}
    \mathcal{X}_i\bigcap\bigcup_{j\in[I]\setminus i}\mathcal{X}_j=\emptyset
\end{equation}
which implies
\begin{equation} \label{eq:RN_union_disjoint_nonempty_open_subsets}
    \left(\mathbb{R}^n\setminus\mathcal{X}_i\right)\bigcup\left(\mathbb{R}^n\setminus\bigcup_{j\in[I]\setminus i}\mathcal{X}_j\right)=\mathbb{R}^n.
\end{equation}
Because the union of any finite collection of closed sets is closed, it must be true that $\bigcup_{j\in[I]\setminus i}\mathcal{X}_j$ is closed. Notice that $\mathcal{X}_i$ is never the whole space $\mathbb{R}^n$ because the CPWL function is assumed to be non-affine. $\bigcup_{j\in[I]\setminus i}\mathcal{X}_j$ must be nonempty due to Definition \ref{def:continuous_piecewise_linear_fn}. 
Therefore, both $\mathbb{R}^n\setminus\mathcal{X}_i$ and $\mathbb{R}^n\setminus\bigcup_{j\in[I]\setminus i}\mathcal{X}_j$ are nonempty and open.
Since $\mathbb{R}^n$ is connected, it cannot be represented as the union of two disjoint nonempty open subsets. It follows that the intersection between $\mathbb{R}^n\setminus\mathcal{X}_i$ and $\mathbb{R}^n\setminus\bigcup_{j\in[I]\setminus i}\mathcal{X}_j$ is nonempty. In other words, there exists an element of $\mathbb{R}^n$ that is not in $\mathcal{X}_i$ and $\bigcup_{j\in[I]\setminus i}\mathcal{X}_j$, contradicting Definition \ref{def:continuous_piecewise_linear_fn}.
\end{proof}

\subsection{Proof of Lemma \ref{lemma:merge1}\ref{lemma:intersecting_nonempty_intersection}} \label{proof:intersecting_nonempty_intersection}
\begin{proof}
If the CPWL function is affine, then there are no intersecting closed subsets because the only closed subset satisfying Assumption \ref{assumption:connected_minimum} is $\mathbb{R}^n$. On the other hand, if the CPWL function is non-affine, then there exist at least two intersecting closed subsets according to Lemma \ref{lemma:existence_intersecting_subsets}.
For any two intersecting closed subsets, say $\mathcal{X}_i$ and $\mathcal{X}_j$, we first show that
\begin{equation}
    \{\mathbf{x}\in\mathbb{R}^n\left|\right.f_i(\mathbf{x})=f_j(\mathbf{x})\}\neq\emptyset
\end{equation}
where $f_i$ and $f_j$ are the affine functions corresponding to $\mathcal{X}_i$ and $\mathcal{X}_j$. We prove this statement by contradiction. Suppose that the intersection is empty, i.e., the linear equation $\left(\mathbf{a}_i-\mathbf{a}_j\right)^T\mathbf{x}+b_i-b_j=0$ does not have a solution where $f_i(\mathbf{x})=\mathbf{a}_i^T\mathbf{x}+b_i$ and $f_j(\mathbf{x})=\mathbf{a}_j^T\mathbf{x}+b_j$ for $\mathbf{a}_i,\mathbf{a}_j\in\mathbb{R}^n$ and $b_i,b_j\in\mathbb{R}$. Then, it is necessary that $\mathbf{a}_i=\mathbf{a}_j$ and $b_i\neq b_j$. In other words, the two affine functions are parallel, implying that every point in $\mathcal{X}_i\cap\mathcal{X}_j$ gives two different values, which cannot be true for a valid function.

Next, we prove that there does not exist an intersection that is $\mathbb{R}^n$ by contradiction. Let us assume that there exists at least one intersection that is $\mathbb{R}^n$ between the affine functions corresponding to two intersecting closed subsets, say $\mathcal{X}_i$ and $\mathcal{X}_j$. Then, we can always replace $\mathcal{X}_i$ and $\mathcal{X}_j$ with their union. Such a replacement still satisfies Definition \ref{def:continuous_piecewise_linear_fn} but reduces the number of closed (connected) subsets by at least one, contradicting the fact that the number of closed subsets is a minimum. Because the two affine functions are identical if and only if the intersection is $\mathbb{R}^n$, the two affine functions must be different.
\end{proof}

\subsection{Proof of Lemma \ref{lemma:merge1}\ref{lemma:intersection_affine_subspace}} \label{proof:intersection_affine_subspace}
\begin{proof}
The claim follows from Lemma \ref{lemma:merge1}\ref{lemma:intersecting_nonempty_intersection}. Because the two affine functions have a nonempty intersection, their intersection must be $\mathbb{R}^n$ or an affine subspace of $\mathbb{R}^n$ with dimension $n-1$. However, the two affine functions must be different, implying that $\mathbb{R}^n$ is never the intersection.
\end{proof}

\subsection{Proof of Lemma \ref{lemma:merge1}\ref{lamma:intersection_included_in_affine_subspace}} \label{proof:intersection_included_in_affine_subspace}
\begin{proof}
Let any given two intersecting subsets be $\mathcal{X}_i$ and $\mathcal{X}_j$. The intersection between their corresponding affine functions, say $f_i$ and $f_j$, is given by
$
    \mathcal{H}_{ij}=\{\mathbf{x}\in\mathbb{R}^n\left|\right.f_i(\mathbf{x})=f_j(\mathbf{x})\}.
$
Suppose that there exists a point $\mathbf{a}\in\mathcal{X}_i\bigcap\mathcal{X}_j$ such that $\mathbf{a}\not\in\mathcal{H}_{ij}$, then it follows that $f_i(\mathbf{a})\neq f_j(\mathbf{a})$. Such a result cannot be true for a valid function. We conclude that $\mathcal{X}_i\bigcap\mathcal{X}_j\subseteq\mathcal{H}_{ij}$.
\end{proof}

\subsection{Proof of Lemma \ref{lemma:merge1}\ref{lemma:intersection_is_boundary}} \label{proof:intersection_is_boundary}
\begin{proof}
We prove the statement by contradiction. Suppose there exists a point $\mathbf{c}\in\mathbb{R}^n$ in the intersection of two intersecting closed connected subsets, say $\mathcal{X}_i$ and $\mathcal{X}_j$, such that $\mathbf{c}$ is an interior point of $\mathcal{X}_i$, then there exists an open $\epsilon$-radius ball $B(\mathbf{c},\epsilon)$ such that
$
    \mathbf{x}\in\mathcal{X}_i, \forall \mathbf{x}\in B(\mathbf{c},\epsilon)
$
for some $\epsilon>0$. By Lemma \ref{lemma:merge1}\ref{lemma:intersection_affine_subspace}, the intersection between the two affine functions corresponding to $\mathcal{X}_i$ and $\mathcal{X}_j$ must be an affine subspace of $\mathbb{R}^n$ with dimension $n-1$. Let such an affine subspace be denoted as $\mathcal{H}_{ij}$ and its corresponding linear subspace be denoted as $\mathcal{V}(\mathcal{H}_{ij})$. Then, there exists a nonzero vector $\mathbf{d}\in\mathbb{R}^n$ such that
$
    \alpha\mathbf{d}\perp\mathbf{v}
$
for all $\mathbf{v}\in\mathcal{V}(\mathcal{H}_{ij})$ and any $\alpha\neq 0$. Therefore, it follows that $\alpha\mathbf{d}+\mathbf{a}\not\in\mathcal{H}_{ij}$ for any $\mathbf{a}\in\mathcal{H}_{ij}$ and any $\alpha\neq 0$. According to Lemma \ref{lemma:merge1}\ref{lamma:intersection_included_in_affine_subspace},
$
    \mathcal{X}_i\cap\mathcal{X}_j\subseteq\mathcal{H}_{ij}
$, so we have $\alpha\mathbf{d}+\mathbf{c}\not\in\mathcal{X}_i\cap\mathcal{X}_j$ for any $\alpha\neq 0$. When $\alpha=\frac{\epsilon}{2\norm{\mathbf{d}}_2}$ or $\alpha=\frac{-\epsilon}{2\norm{\mathbf{d}}_2}$, $\alpha\mathbf{d}+\mathbf{c}\in B(\mathbf{c},\epsilon)$. However, one of them must satisfy $\alpha\mathbf{d}+\mathbf{c}\not\in\mathcal{X}_i$, contradicting the existence of a point in $\mathcal{X}_i\cap\mathcal{X}_j$ that is an interior point of $\mathcal{X}_i$. The same procedure can be applied to prove that there does not exist a point in $\mathcal{X}_i\cap\mathcal{X}_j$ such that it is an interior point of $\mathcal{X}_j$. We conclude that every element in $\mathcal{X}_i\cap\mathcal{X}_j$ is not an interior point of $\mathcal{X}_i$ or $\mathcal{X}_j$.
\end{proof}

\subsection{Proof of Lemma \ref{lemma:merge2}\ref{lemma:nonempty_interior}} \label{proof:nonempty_interior}
\begin{proof}
The boundary or frontier of $\mathcal{X}_i$ is given by
\begin{equation} \label{eq:boundary_Xi}
    \begin{split}
        \text{Fr}\mathcal{X}_i&=\overline{\mathcal{X}_i}\bigcap\overline{\mathbb{R}^n\setminus\mathcal{X}_i}\\
        &=\mathcal{X}_i\bigcap\overline{\left(\bigcup_{k\in[I]}\mathcal{X}_k\right)\setminus\mathcal{X}_i}\\
        &=\mathcal{X}_i\bigcap\overline{\bigcup_{k\in[I]\setminus i}\left(\mathcal{X}_k\setminus\mathcal{X}_k\bigcap\mathcal{X}_i\right)}\\
        &=\mathcal{X}_i\bigcap\bigcup_{k\in[I]\setminus i}\overline{\left(\mathcal{X}_k\setminus\mathcal{X}_k\bigcap\mathcal{X}_i\right)}\\
        &=\mathcal{X}_i\bigcap\bigcup_{k\in[I]\setminus i}\overline{\mathcal{X}_k}\\
        &=\mathcal{X}_i\bigcap\bigcup_{k\in[I]\setminus i}\mathcal{X}_k\\
        &=\bigcup_{k\in[I]\setminus i}\mathcal{X}_k\bigcap\mathcal{X}_i\\
    \end{split}
\end{equation}
where $\overline{\mathcal{A}}$ denotes the closure of a subset $\mathcal{A}$. We have used Lemma \ref{lemma:merge1}\ref{lemma:intersection_is_boundary} for the equality between the $4$-th and $5$-th line of (\ref{eq:boundary_Xi}). Now, we prove that the interior of $\mathcal{X}_i$ is nonemtpy by contradiction. Suppose that the interior of $\mathcal{X}_i$ is empty, then it follows that $\mathcal{X}_i=\overline{\mathcal{X}_i}=\text{Fr}\mathcal{X}_i$ because the closure of $\mathcal{X}_i$ is the union of the interior and the boundary of $\mathcal{X}_i$. Combining that with (\ref{eq:boundary_Xi}), we have
$
    \mathcal{X}_i=\bigcup_{k\in[I]\setminus i}\mathcal{X}_k\bigcap\mathcal{X}_i.
$
which implies every element in $\mathcal{X}_i$ is at least covered by one of the other closed subsets $\mathcal{X}_k$ for some $k\in[I]\setminus i$. In this case, we can delete $\mathcal{X}_i$ from $\mathcal{X}_1,\mathcal{X}_2,\cdots,\mathcal{X}_I$; and the remaining $I-1$ closed subsets still satisfy Definition \ref{def:continuous_piecewise_linear_fn}. Such a valid deletion of $\mathcal{X}_i$ contradicts the fact that $I$ is the minimum number of closed subsets. Hence, the interior of $\mathcal{X}_i$ must be nonempty.
\end{proof}

\subsection{Proof of Lemma \ref{lemma:merge2}\ref{lemma:boundary_union_other_subsets}} \label{proof:boundary_union_other_subsets}
\begin{proof}
The statement is proved by (\ref{eq:boundary_Xi}) in Lemma \ref{lemma:merge2}\ref{lemma:nonempty_interior}.
\end{proof}

\subsection{Proof of Lemma \ref{lemma:merge2}\ref{lemma:nonempty_and_disjoint_interiors}} \label{proof:nonempty_and_disjoint_interiors}
\begin{proof}
    By Lemma \ref{lemma:merge2}\ref{lemma:nonempty_interior}, the interior of every subset is nonempty. Next, by Lemma \ref{lemma:merge1}\ref{lemma:intersection_is_boundary}, every point in the intersection between any two subsets is a boundary point of both subsets. It follows that the interiors of any two subsets are disjoint. 
\end{proof}

\subsection{Proof of Lemma \ref{lemma:subsets_cannot_be_included_union_affine_subspaces}} \label{proof:subsets_cannot_be_included_union_affine_subspaces}
\begin{proof}
    By Lemma \ref{lemma:merge2}\ref{lemma:nonempty_interior}, the interior of $\mathcal{X}_i$ is nonempty. Therefore, there exists an open $\epsilon$-radius ball $B(\mathbf{c}_0,\epsilon)$ such that $\mathbf{x}\in\mathcal{X}_i,\forall \mathbf{x}\in B(\mathbf{c}_0,\epsilon)$ for some $\epsilon>0$ and $\mathbf{c}_0\in\mathcal{X}_i$.
    Let us consider the set
    \begin{equation} \label{eq:intersections_open_half_spaces}
        \bigcap_{j\in[k]} \left(B(\mathbf{c}_0,\epsilon)\bigcap\left(\mathcal{H}^+_{j}\bigcup\mathcal{H}^-_{j}\right)\right)
    \end{equation}
    where $\mathcal{H}^+_{j}$ and $\mathcal{H}^-_{j}$ are two open half spaces created by $\mathcal{H}_{j}$.
    It suffices to show the nonemptyness of the set in (\ref{eq:intersections_open_half_spaces}) to prove the claim.
    If $\mathcal{H}_{j}$ and $B(\mathbf{c}_0,\epsilon)$ do not intersect, then $B(\mathbf{c}_0,\epsilon)$ completely belongs to $\mathcal{H}^+_{j}$ or $\mathcal{H}^-_{j}$. Without loss of generality, we can remove all $j$ such that $\mathcal{H}_{j}$ does not intersect $B(\mathbf{c}_0,\epsilon)$ and assume there are $k$ affine subspaces of $\mathbb{R}^n$ with dimension $n-1$ intersecting $B(\mathbf{c}_0,\epsilon)$. Let us sequentially carry out the intersection in (\ref{eq:intersections_open_half_spaces}). Every time before the operation of the $j$-th intersection between $B(\mathbf{c}_{j-1},\frac{\epsilon}{2^{j-1}})$ and $\left(\mathcal{H}^+_{j}\bigcup\mathcal{H}^-_{j}\right)$, there exists an open $\frac{\epsilon}{2^{j}}$-radius ball $B(\mathbf{c}_j,\frac{\epsilon}{2^{j}})$ for some $\mathbf{c}_j\in B(\mathbf{c}_{j-1},\frac{\epsilon}{2^{j-1}})$ such that it does not intersect with $\mathcal{H}_j$. Therefore, at the end of the sequential process, there exists an open ball that does not intersect any of these $k$ affine subspaces of $\mathbb{R}^n$ with dimension $n-1$. The set in (\ref{eq:intersections_open_half_spaces}) is nonempty, implying (\ref{eq:subset_not_included_union_finite}) holds true.
\end{proof}

\subsection{Proof of Proposition \ref{proposition:largest_closed_connected_minimum}} \label{proof:largest_closed_connected_minimum}
\begin{proof}
We prove the claim by contraposition. If the number of closed connected subsets is not a minimum, i.e., Assumption \ref{assumption:connected_minimum} is not satisfied, then such a number can be decreased by merging some of the intersecting closed connected subsets that have the same corresponding affine functions. Therefore, there exist at least two closed connected subsets that can be made larger.

On the other hand, if the closed connected subsets, say $\mathcal{X}_1,\mathcal{X}_2,\cdots,\mathcal{X}_I$, have at least one of the subsets that can be made larger, then there exist at least two intersecting closed connected subsets, say $\mathcal{X}_i$ and $\mathcal{X}_j$, from $\mathcal{X}_1,\mathcal{X}_2,\cdots,\mathcal{X}_I$ such that their corresponding affine functions are the same. Otherwise, any closed connected subset cannot be made larger than itself.
Therefore, $\mathcal{X}_i$ and $\mathcal{X}_j$ can be replaced with $\mathcal{X}_i\bigcup\mathcal{X}_j$ and these $I-1$ closed connected subsets still satisfy Definition \ref{def:continuous_piecewise_linear_fn}, implying that $I$ is not the minimum.
\end{proof}
\section{Algorithms and time complexities} \label{algorithm}
\begin{table}[H]
    \caption{The running time of Algorithm \ref{alg:find_a_network_for_a_piecewise_linear_fn} is upper bounded by $\mathsf{poly}(n,k,q,L)$.}
    \centering
    \begin{tabular}{ccc}
        \toprule
        Line & Operation count & Explanation\\
        \midrule
        1 & $\mathcal{O}\left(nq\max(n^2,q)\right)$ & Algorithm \ref{alg:find_all_distinct_linear_components} (see Table \ref{tab:find_all_distinct_linear_components}).\\
        2 & $\mathcal{O}(q)$ & Repeat Line 3 to Line 9 $q$ times.\\
        3 & $\mathcal{O}(1)$ & Initialize an empty placeholder.\\
        4 & $\mathcal{O}(k)$ & Repeat Line 5 to Line 7 $k$ times.\\
        5 & $\mathsf{poly}\left(n,q,L\right)$ & Solve a linear program \citep{vavasis1996primal}.\\
        6 & $\mathcal{O}(1)$ & Add an index.\\
        7 & - & -\\
        8 & - & -\\
        9 &  $\mathcal{O}\left(k^2\max(k\log_2k,n)\right)$ & Algorithm \ref{alg:find_a_extremum_ReLU_network} (see Table \ref{tab:find_a_extremum_ReLU_network}).\\
        10 & - & -\\
        11 & $\mathcal{O}\left(q\max(n,k)^2\max(n,k,q)\log_2k\right)$ & Algorithm \ref{alg:find_a_concat_ReLU_network} (see Table \ref{tab:find_a_concat_ReLU_network}).\\
        12 & $\mathcal{O}\left(q^3\log_2q\right)$ & Algorithm \ref{alg:find_a_extremum_ReLU_network} (see Table \ref{tab:find_a_extremum_ReLU_network}).\\
        13 & $\mathcal{O}\left(q^3\max(n,k)^3\log_2q\right)$ & Algorithm \ref{alg:find_a_composited_ReLU_network} (see Table \ref{tab:find_a_composited_ReLU_network}).\\
        \bottomrule
    \end{tabular}
    \label{tab:find_a_network_for_a_piecewise_linear_fn}
\end{table}
\begin{algorithm}
\caption{Find a ReLU network that computes the extremum of affine functions}\label{alg:find_a_extremum_ReLU_network}
\begin{algorithmic}[1]
\Require Scalar-valued affine functions $f_1,f_2,\cdots,f_m$ on $\mathbb{R}^n$ and the type of extremum (max or min).
\Ensure Parameters of an $l$-layer ReLU network $g$ computing $g(\mathbf{x})=\max_{i\in[m]}. f_i(\mathbf{x})$ or $g(\mathbf{x})=\min_{i\in[m]}. f_i(\mathbf{x})$ for all $\mathbf{x}\in\mathbb{R}^n$.
\State $\mathbf{A} \gets \begin{bmatrix}-1&1\\1&0\\-1&0\end{bmatrix}, \mathbf{B}\gets\begin{bmatrix}1&1&-1\end{bmatrix},\mathbf{C}\gets\begin{bmatrix}1\\-1\end{bmatrix}$ \Comment{Constant matrices}
\State $\boldsymbol{\Psi}(\mathbf{Y},\mathbf{Z})\gets\begin{bmatrix}\mathbf{Y}&\mathbf{0}\\\mathbf{0}&\mathbf{Z}\end{bmatrix}$ \Comment{A function generating a block diagonal matrix composed of $\mathbf{Y}$ and $\mathbf{Z}$}
\State $\boldsymbol{\Phi}(\mathbf{Y},s)\gets\begin{bmatrix}\mathbf{Y}^{(1)}&\mathbf{0}&\cdots&\mathbf{0}\\\mathbf{0}&\mathbf{Y}^{(2)}&\cdots&\mathbf{0}\\\vdots&\mathbf{\vdots}&\ddots&\vdots\\\mathbf{0}&\mathbf{0}&\cdots&\mathbf{Y}^{(s)}\end{bmatrix}$ \Comment{A block diagonal matrix with $\mathbf{Y}$ repeated $s$ times}
\State $l \gets \lceil\log_2m\rceil+1, k_0 \gets n, k_{l} \gets 1, c_0 \gets m$ \Comment{$l$ is the number of layers of $g$}
\For{$i=1,2,\cdots,l-1$}
    \If{$c_{i-1}$ is even}
        \State $c_i \gets \frac{c_{i-1}}{2}$
        \State $k_i \gets 3c_{i}$ \Comment{Output dimension of the $i$-th layer}
    \Else
        \State $c_i \gets \frac{c_{i-1}+1}{2}$
        \State $k_i \gets 3c_i-1$ \Comment{Output dimension of the $i$-th layer}
    \EndIf
\EndFor
\State $\mathbf{W}_1 \gets \begin{bmatrix}\nabla f_1&\nabla f_2&\cdots&\nabla f_m\end{bmatrix}^\mathsf{T},\mathbf{b}_1 \gets\begin{bmatrix}f_1(0)&f_2(0)&\cdots&f_m(0)\end{bmatrix}^\mathsf{T}$
\If{$l>1$} \Comment{Find the weights of input and output layers, if any}
\If{$c_0$ is even}
\State $\mathbf{W}_1 \gets \boldsymbol{\Phi}\left(\mathbf{A},c_1\right)\mathbf{W}_1,\mathbf{b}_1 \gets \boldsymbol{\Phi}\left(\mathbf{A},c_1\right)\mathbf{b}_1$
\Else
\State $\mathbf{W}_1 \gets \boldsymbol{\Psi}\left(\boldsymbol{\Phi}\left(\mathbf{A},c_1-1\right),\mathbf{C}\right)\mathbf{W}_1,\mathbf{b}_1 \gets \boldsymbol{\Psi}\left(\boldsymbol{\Phi}\left(\mathbf{A},c_1-1\right),\mathbf{C}\right)\mathbf{b}_1$
\EndIf
\State $\mathbf{W}_l \gets \mathbf{B}, \mathbf{b}_l \gets \mathbf{0}_{k_l}$
\EndIf
\If{$l>2$} \Comment{Find the weights of remaining layers, if any}
\For{$i=2,3,\cdots,l-1$}
    \If{$c_{i-1}$ is even}
        \State $\mathbf{T} \gets \boldsymbol{\Phi}\left(\mathbf{A},c_{i}\right)$
    \Else
        \State $\mathbf{T} \gets \boldsymbol{\Psi}\left(\boldsymbol{\Phi}\left(\mathbf{A},c_{i}-1\right),\mathbf{C}\right)$
    \EndIf
    \If{$c_{i-2}$ is even}
        \State $\mathbf{W}_i \gets \mathbf{T}\boldsymbol{\Phi}\left(\mathbf{B},c_{i-1}\right)$
    \Else
        \State $\mathbf{W}_i \gets \mathbf{T}\boldsymbol{\Psi}\left(\boldsymbol{\Phi}\left(\mathbf{B},c_{i-1}-1\right),\mathbf{C}^{\mathsf{T}}\right)$
    \EndIf
    \State $\mathbf{b}_i \gets \mathbf{0}_{k_i}$
\EndFor
\EndIf
\If{type of extremum is the minimum}
\State $\mathbf{W}_1 \gets -\mathbf{W}_1, \mathbf{b}_1 \gets -\mathbf{b}_1$
\State $\mathbf{W}_l \gets -\mathbf{W}_l, \mathbf{b}_l \gets -\mathbf{b}_l$
\EndIf \Comment{See Table \ref{tab:find_a_extremum_ReLU_network} in Appendix \ref{algorithm} for complexity analysis}
\end{algorithmic}
\end{algorithm}

\begin{algorithm}
\caption{Find a ReLU network that concatenates a number of given ReLU networks}\label{alg:find_a_concat_ReLU_network}
\begin{algorithmic}[1]
\Require Weights of $k$ ReLU networks $g_1,g_2,\cdots,g_k$ denoted by $\{\mathbf{W}_i^j,\mathbf{b}_i^j\}_{i=1}^{l_j}$ for $j\in[k]$.
\Ensure Parameters of an $l$-layer ReLU network $g$ computing $g(\mathbf{x})=\begin{bmatrix}g_1(\mathbf{x})\\g_2(\mathbf{x})\\\vdots\\g_k(\mathbf{x})\end{bmatrix},\forall \mathbf{x}\in\mathbb{R}^n$.
\State $l \gets \max_{j\in[k]}l_j$
\State $\mathbf{W}_1 \gets \begin{bmatrix}\mathbf{W}_1^1\\\mathbf{W}_1^2\\\vdots\\\mathbf{W}_1^k\end{bmatrix},\mathbf{b}_1 \gets \begin{bmatrix}\mathbf{b}_1^1\\\mathbf{b}_1^2\\\vdots\\\mathbf{b}_1^k\end{bmatrix}$ \Comment{Weights of the input layer}
\For{$j=1,2,\cdots,k$} 
    \If{$l_j<l$} \Comment{Append an identity mapping network to the network if it is shallower}
    \State $m \gets $ output dimsion of $g_j$
    \State $g_j^c \gets $ run Algorithm \ref{alg:find_an_identity_ReLU_network} with an input dimension $m$ and a number of layers $l-l_j+1$
    \State $g_j' \gets $ run Algorithm \ref{alg:find_a_composited_ReLU_network} with $g_j$ and $g_j^c$
    \State $\{\mathbf{W}_i^j,\mathbf{b}_i^j\}_{i=1}^{l} \gets $ weights of $g_j'$
    \EndIf
\EndFor
\For{$i=2,3,\cdots,l$} \Comment{Find the remaining weights}
    \State $\mathbf{W}_i \gets \begin{bmatrix}\mathbf{W}_i^1&\mathbf{0}&\cdots&\mathbf{0}\\\mathbf{0}&\mathbf{W}_i^2&\cdots&\mathbf{0}\\\vdots&\mathbf{\vdots}&\ddots&\vdots\\\mathbf{0}&\mathbf{0}&\cdots&\mathbf{W}_i^k\end{bmatrix}, \mathbf{b}_i \gets \begin{bmatrix}\mathbf{b}_i^1\\\mathbf{b}_i^2\\\vdots\\\mathbf{b}_i^k\end{bmatrix}$
\EndFor \Comment{See Table \ref{tab:find_a_concat_ReLU_network} in Appendix \ref{algorithm} for complexity analysis}
\end{algorithmic}
\end{algorithm}

\begin{algorithm}
\caption{Find a ReLU network computing a composition of two given ReLU networks}\label{alg:find_a_composited_ReLU_network}
\begin{algorithmic}[1]
\Require Weights of two ReLU networks $g_1$ and $g_2$ denoted by $\{\mathbf{W}_i^1,\mathbf{b}_i^1\}_{i=1}^{l_1}$ and $\{\mathbf{W}_i^2,\mathbf{b}_i^2\}_{i=1}^{l_2}$.
\Ensure Parameters of an $l$-layer ReLU network $g$ computing $g(\mathbf{x})=g_2\left(g_1(\mathbf{x})\right),\forall \mathbf{x}\in\mathbb{R}^n$.
\State $l \gets l_1+l_2-1$
\For{$i=1,2,\cdots,l$}
    \If{$i<l_1$} \Comment{The first $l_1-1$ layers are identical to the corresponding layers in $g_1$}
        \State $\mathbf{W}_i \gets \mathbf{W}_i^1, \mathbf{b}_i \gets \mathbf{b}_i^1$
    \ElsIf{$i=l_1$} \Comment{A composition of affine functions is still an affine function}
        \State $\mathbf{W}_i \gets \mathbf{W}_1^2\mathbf{W}_{l_1}^1, \mathbf{b}_i \gets \mathbf{W}_1^2\mathbf{b}_{l_1}^1+\mathbf{b}_{1}^2$
    \Else  \Comment{The last $l_2-1$ layers are identical to the corresponding layers in $g_2$}
        \State $\mathbf{W}_i \gets \mathbf{W}_{i-l_1+1}^2, \mathbf{b}_i \gets \mathbf{b}_{i-l_1+1}^2$
    \EndIf
\EndFor \Comment{See Table \ref{tab:find_a_composited_ReLU_network} in Appendix \ref{algorithm} for complexity analysis}
\end{algorithmic}
\end{algorithm}

\begin{algorithm}
\caption{Find a ReLU network that computes an identity mapping for a given depth}\label{alg:find_an_identity_ReLU_network}
\begin{algorithmic}[1]
\Require The input dimension $n$ and the number of layers $l$ of the target ReLU network.
\Ensure Parameters of an $l$-layer ReLU network $g$ computing $g(\mathbf{x})=\mathbf{x},\forall \mathbf{x}\in\mathbb{R}^n$.
\State $\mathbf{A} \gets \begin{bmatrix}1\\-1\end{bmatrix}, \mathbf{B}\gets\begin{bmatrix}1&-1\end{bmatrix},\mathbf{C}\gets\begin{bmatrix}1&-1\\-1&1\end{bmatrix}$ \Comment{Constant matrices}
\State $\boldsymbol{\Phi}(\mathbf{Y},s)=\begin{bmatrix}\mathbf{Y}^{(1)}&\mathbf{0}&\cdots&\mathbf{0}\\\mathbf{0}&\mathbf{Y}^{(2)}&\cdots&\mathbf{0}\\\vdots&\mathbf{\vdots}&\ddots&\vdots\\\mathbf{0}&\mathbf{0}&\cdots&\mathbf{Y}^{(s)}\end{bmatrix}$ \Comment{A block diagonal matrix with $\mathbf{Y}$ repeated $s$ times}
\State $k_0 \gets n, k_l \gets n, \mathbf{b}_l\gets\mathbf{0}_{n}$
\For{$i=1,2,\cdots,l-1$}
    \State $k_i \gets 2n$ \Comment{The number of hidden neurons at the $i$-th hidden layer}
    \State $\mathbf{b}_{i} \gets \mathbf{0}_{k_{i}}$
\EndFor
\If{$l=1$} \Comment{Find the weights of input and output layers, if any}
    \State $\mathbf{W}_1 \gets \mathbf{I}_{n\times n}$ \Comment{An identity matrix}
\Else
    \State $\mathbf{W}_1 \gets \boldsymbol{\Phi}\left(\mathbf{A},k_0\right)$
    \State $\mathbf{W}_l \gets \boldsymbol{\Phi}\left(\mathbf{B},k_l\right)$
\EndIf
\If{$l>2$} \Comment{Find the weights of hidden layers, if any}
    \For{$i=2,3,\cdots,l-1$}
        \State $\mathbf{W}_i \gets \boldsymbol{\Phi}\left(\mathbf{C},n\right)$
    \EndFor
\EndIf \Comment{See Table \ref{tab:find_an_identity_ReLU_network} in Appendix \ref{algorithm} for complexity analysis}
\end{algorithmic}
\end{algorithm}

\begin{algorithm}
\caption{Find all distinct linear components of a CPWL function}\label{alg:find_all_distinct_linear_components}
\begin{algorithmic}[1]
\Require An unknown CPWL function $p$ whose output can be observed by feeding input from $\mathbb{R}^n$ to the function. A center $\mathbf{c}_i$ and radius $\epsilon_i>0$ of any closed $\epsilon_i$-radius ball $B(\mathbf{c}_i,\epsilon_i)$ such that $B(\mathbf{c}_i,\epsilon_i)\subset\mathcal{X}_i$ for $i=1,2,\cdots,q$ where $\{\mathcal{X}_i\}_{i\in[q]}$ are all pieces of $p$.
\Ensure All distinct linear components of $p$, denoted by $\mathcal{F}$.
\State $\mathcal{F}\gets\emptyset$ \Comment{Initialize the set of all distinct linear components}
\For{$i=1,2,\cdots,q$}
    \State $\mathbf{x}_0 \gets \mathbf{c}_i$  \Comment{select the center of $B(\mathbf{c}_i,\epsilon_i)$}
    \State $y_0 \gets p(\mathbf{x}_0)$
    \State $\begin{bmatrix}\mathbf{s}_1&\mathbf{s}_2&\cdots&\mathbf{s}_n\end{bmatrix} \gets \epsilon_i\mathbf{I}_{n\times n}$ \Comment{scale the standard basis of $\mathbb{R}^n$}
    \State $\mathbf{S} \gets \begin{bmatrix}\mathbf{s}_1&\mathbf{s}_2&\cdots&\mathbf{s}_n\end{bmatrix}$
    \State $\mathbf{z}\gets\begin{bmatrix}p(\mathbf{s}_1+\mathbf{x}_0)-y_0\\p(\mathbf{s}_2+\mathbf{x}_0)-y_0\\\vdots\\p(\mathbf{s}_n+\mathbf{x}_0)-y_0\end{bmatrix}$
    \State $\mathbf{a} \gets \mathbf{S}^{-\mathsf{T}}\mathbf{z}$ \Comment{Find the linear map by solving a system of linear equations}
    \State $b \gets y_0-\mathbf{a}^{\mathsf{T}}\mathbf{x}_0$  \Comment{Find the translation}
    \State $f \gets \mathbf{x}\mapsto\mathbf{a}^{\mathsf{T}}\mathbf{x}+b$ \Comment{The affine map on $\mathcal{X}_i$}
    \If{$f\not\in\mathcal{F}$} \Comment{Only add the affine map $f$ to the set $\mathcal{F}$ if $f$ is distinct to all elements of $\mathcal{F}$}
        \State $\mathcal{F} \gets \mathcal{F}\bigcup\{f\}$
    \EndIf
\EndFor \Comment{See Table \ref{tab:find_all_distinct_linear_components} in Appendix \ref{algorithm} for complexity analysis}
\end{algorithmic}
\end{algorithm}

\begin{table}[H]
    \caption{The time complexity of Algorithm \ref{alg:find_a_extremum_ReLU_network} is $\mathcal{O}\left(m^2\max(m\log_2m,n)\right)$.}
    \centering
    \begin{tabular}{ccc}
        \toprule
        Line & Operation count & Explanation\\
        \midrule
        1 & $\mathcal{O}(1)$ & Initialize constant matrices.\\
        2 & $\mathcal{O}\left(d_1^2\right)$ & Let $d_1$ be the maximum dimension of $\mathbf{Y}$ and $\mathbf{Z}$.\\
        3 & $\mathcal{O}(s^2d_2^2)$ & Let $d_2$ be the maximum dimension of $\mathbf{Y}$.\\
        4 & $\mathcal{O}(1)$ & Scalar assignments.\\
        5 & $\mathcal{O}(\log_2m)$ & Repeat Line 6 to Line 12 $\lceil\log_2m\rceil$ times.\\
        6 & $\mathcal{O}(1)$ & Check a scalar is even or not.\\
        7 & $\mathcal{O}(1)$ & Compute a scalar.\\
        8 & $\mathcal{O}(1)$ & Compute a scalar.\\
        9 & - & -\\
        10 & $\mathcal{O}(1)$ & Compute a scalar.\\
        11 & $\mathcal{O}(1)$ & Compute a scalar.\\
        12 & - & -\\
        13 & - & -\\
        14 & $\mathcal{O}(mn)$ & Assign a matrix and a vector.\\
        15 & $\mathcal{O}(1)$ & Check a scalar inequality.\\
        16 & $\mathcal{O}(1)$ & Check a scalar is even or not.\\
        17 & $\mathcal{O}(m^2n)$ & Matrix creation and multiplication.\\
        18 & - & -\\
        19 & $\mathcal{O}(m^2n)$ & Matrix creation and multiplication.\\
        20 & - & -\\
        21 & $\mathcal{O}(1)$ & Assign a constant matrix and vector.\\
        22 & - & -\\
        23 & $\mathcal{O}(1)$ & Check a scalar inequality.\\
        24 & $\mathcal{O}(\log_2m)$ & Repeat Line 25 to Line 30 $\lceil\log_2m\rceil-1$ times.\\
        25 & $\mathcal{O}(1)$ & Check a scalar is even or not.\\
        26 & $\mathcal{O}(m^2)$ & Matrix creation.\\
        27 & - & -\\
        28 & $\mathcal{O}(m^2)$ & Matrix creation.\\
        29 & - & -\\
        30 & $\mathcal{O}(1)$ & Check a scalar is even or not.\\
        31 & $\mathcal{O}(m^3)$ & Matrix creation and multiplication.\\
        32 & - & -\\
        33 & $\mathcal{O}(m^3)$ & Matrix creation and multiplication.\\
        34 & - & -\\
        35 & $\mathcal{O}(m)$ & Assign a vector whose length is at most $\left\lceil\frac{3m}{2}\right\rceil$.\\
        36 & - & -\\
        37 & - & -\\
        38 & $\mathcal{O}(1)$ & Check the binary data type.\\
        39 & $\mathcal{O}(mn)$ & Reverse the sign of $\mathbf{W}_1$ and $\mathbf{b}_1$.\\
        40 & $\mathcal{O}(1)$ & Reverse the sign of a constant matrix and a constant bias.\\
        41 & - & -\\
        \bottomrule
    \end{tabular}
    \label{tab:find_a_extremum_ReLU_network}
\end{table}

\begin{table}[H]
    \caption{The time complexity of Algorithm \ref{alg:find_a_concat_ReLU_network} is $\mathcal{O}\left(d^2kl\max(d,k)\right)$ where $d$ is the maximum dimension of all the weight matrices in $g_1,g_2,\cdots,g_k$ and $l=\max_{j\in[k]}l_j$.}
    \centering
    \begin{tabular}{ccc}
        \toprule
        Line & Operation count & Explanation\\
        \midrule
        1 & $\mathcal{O}(k)$ & Find the maximum among $k$ numbers.\\
        2 & $\mathcal{O}(d^2k)$ & Matrix concatenation and assignment.\\
        3 & $\mathcal{O}(k)$ & Repeat Line 4 to Line 9 $k$ times.\\
        4 & $\mathcal{O}(1)$ & Check a scalar inequality.\\
        5 & $\mathcal{O}(1)$ & A scalar assignment.\\
        6 & $\mathcal{O}(d^2l)$ & Algorithm \ref{alg:find_an_identity_ReLU_network} (see Table \ref{tab:find_an_identity_ReLU_network}).\\
        7 & $\mathcal{O}(d^3l)$ & Algorithm \ref{alg:find_a_composited_ReLU_network} (see Table \ref{tab:find_a_composited_ReLU_network}).\\
        8 & $\mathcal{O}(d^2l)$ & Assign weights of the network.\\
        9 & - & -\\
        10 & - & -\\
        11 & $\mathcal{O}(l)$ & Repeat Line 12 $l-1$ times.\\
        12 & $\mathcal{O}(d^2k^2)$ & Assign a matrix and a vector.\\
        13 & - & -\\
        \bottomrule
    \end{tabular}
    \label{tab:find_a_concat_ReLU_network}
\end{table}

\begin{table}[H]
    \caption{The time complexity of Algorithm \ref{alg:find_a_composited_ReLU_network} is $\mathcal{O}\left(d^3\max(l_1,l_2)\right)$ where $d$ is the maximum dimension of all the weight matrices in $g_1$ and $g_2$.
    }
    \centering
    \begin{tabular}{ccc}
        \toprule
        Line & Operation count & Explanation\\
        \midrule
        1 & $\mathcal{O}(1)$ & Assign a constant.\\
        2 & $\mathcal{O}(l)$ & Repeat Line 3 to Line 9 $l$ times.\\
        3 & $\mathcal{O}(1)$ & Check a scalar inequality.\\
        4 & $\mathcal{O}(d^2)$ & Assign a matrix and a vector (at most $d^2+d$ elements).\\
        5 & $\mathcal{O}(1)$ & Check a scalar equality.\\
        6 & $\mathcal{O}(d^3)$ & Matrix multiplication and assignment.\\
        7 & - & -\\
        8 & $\mathcal{O}(d^2)$ & Assign a matrix and a vector  (at most $d^2+d$ elements).\\
        9 & - & -\\
        10 & - & -\\
        \bottomrule
    \end{tabular}
    \label{tab:find_a_composited_ReLU_network}
\end{table}

\begin{table}[H]
    \caption{The time complexity of Algorithm \ref{alg:find_an_identity_ReLU_network} is $\mathcal{O}(n^2l)$.}
    \centering
    \begin{tabular}{ccc}
        \toprule
        Line & Operation count & Explanation\\
        \midrule
        1 & $\mathcal{O}(1)$ & Initialize constant matrices.\\
        2 & $\mathcal{O}(s^2d_1d_2)$ & Create a block diagonal matrix from $\mathbf{Y}\in\mathbb{R}^{d_1\times d_2}$ and $s\in\mathbb{N}$.\\
        3 & $\mathcal{O}(n)$ & Assign two constant scalars and one constant vector of length $n$.\\
        4 & $\mathcal{O}(l)$ & Repeat Line 5 to line 6 $l$ times.\\
        5 & $\mathcal{O}(1)$ & Assign a scalar.\\
        6 & $\mathcal{O}(n)$ & Assign a vector whose length $k_i$ is equal to $2n$.\\
        7 & - & -\\
        8 & $\mathcal{O}(1)$ & Check a scalar equality.\\
        9 & $\mathcal{O}(n^2)$ & Assign an $n$-by-$n$ matrix.\\
        10 & - & -\\
        11 & $\mathcal{O}(n^2)$ & Assign a $2n$-by-$n$ block diagonal matrix.\\
        12 & $\mathcal{O}(n^2)$ & Assign an $n$-by-$2n$ block diagonal matrix.\\
        13 & - & -\\
        14 & $\mathcal{O}(1)$ & Check a scalar inequality.\\
        15 & $\mathcal{O}(l)$ & Repeat Line 16 $l-2$ times.\\
        16 & $\mathcal{O}(n^2)$ & Assign a $2n$-by-$2n$ block diagonal matrix.\\
        17 & - & -\\
        18 & - & -\\
        \bottomrule
    \end{tabular}
    \label{tab:find_an_identity_ReLU_network}
\end{table}

\begin{table}[H]
    \caption{The time complexity of Algorithm \ref{alg:find_all_distinct_linear_components} is $\mathcal{O}\left(nq\max(n^2,q)\right)$.}
    \centering
    \begin{tabular}{ccc}
        \toprule
        Line & Operation count & Explanation\\
        \midrule
        1 & $\mathcal{O}(1)$ & Initialize an empty placeholder $\mathcal{F}$.\\
        2 & $\mathcal{O}(q)$ & Repeat Line 3 to line 13 $q$ times.\\
        3 & $\mathcal{O}(1)$ & Select an interior point. Use the center of the ball.\\
        4 & $\mathcal{O}(1)$ & Evaluate the function on the point.\\
        5 & $\mathcal{O}(n^2)$ & Scale and assign an $n$-by-$n$ matrix.\\
        6 & - & -\\
        7 & $\mathcal{O}(n)$ & Translate, evaluate, and subtract $n$ points.\\
        8 & $\mathcal{O}(n^3)$ & Solve a system of $n$ linear equations with $n$ variables.\\
        9 & $\mathcal{O}(n)$ & Solve the translation term in the affine map\\
        10 & - & -\\
        11 & $\mathcal{O}(nq)$ & Each affine map has $n+1$ parameters and $\mathcal{F}$ has at most $q$ elements.\\
        12 & $\mathcal{O}(1)$ & Add a distinct affine map to $\mathcal{F}$.\\
        13 & - & -\\
        14 & - & -\\
        \bottomrule
    \end{tabular}
    \label{tab:find_all_distinct_linear_components}
\end{table}
\section{Open source implementation and run time of Algorithm \ref{alg:find_a_network_for_a_piecewise_linear_fn}} \label{opensource_and_runtime}
We implement Algorithm \ref{alg:find_a_network_for_a_piecewise_linear_fn} in Python. Figure \ref{fig:run_time} shows that the run time of the algorithm is greatly affected by the number of pieces $q$.

Code is available at \href{https://github.com/kjason/CPWL2ReLUNetwork}{https://github.com/kjason/CPWL2ReLUNetwork}.

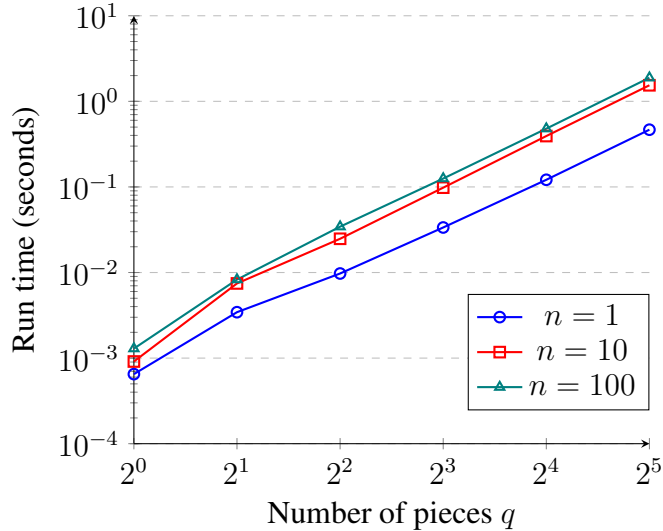
\begin{figure}[h]
\centering
\begin{tikzpicture}
\begin{axis}[
    font=\large,
    axis lines = left,
    xlabel = Number of pieces $q$,
    ylabel = Run time (seconds),
    legend style={at={(1.0,0.35)}},
    ymode=log,
    xmode=log,
    log basis x={2},
    xmax = 32,
    xtick={1,2,4,8,16,32},
    ymajorgrids=true,
    grid style=dashed,
    ymin = 0.0001,
    ymax= 10
]
\addplot [
    color=blue,
    mark=o,
    thick,
]
coordinates {
(1,0.000650125)
(2,0.003439136)
(4,0.009739432)
(8,0.033589473)
(16,0.121194372)
(32,0.466290045)
};
\addlegendentry{$n=1$}
\addplot [
    color=red,
    mark=square,
    thick
]
coordinates {
(1,0.000912652)
(2,0.007454562)
(4,0.02478827)
(8,0.098043351)
(16,0.393176675)
(32,1.540794063)
};
\addlegendentry{$n=10$}
\addplot[
    color=teal,
    mark=triangle,
    thick,
    ]
    coordinates {
(1,0.001291566)
(2,0.00827076)
(4,0.034329739)
(8,0.12526372)
(16,0.479531198)
(32,1.885979624)
    };
\addlegendentry{$n=100$}
\end{axis}
\end{tikzpicture}
\caption{The run time of Algorithm \ref{alg:find_a_network_for_a_piecewise_linear_fn} is an average of $50$ trials. Every trial runs Algorithm \ref{alg:find_a_network_for_a_piecewise_linear_fn} with a random CPWL function whose input dimension is $n$ and number of pieces is $q$. The code provided in the above link is run on a computer (Microsoft Surface Laptop Studio) with the Intel Core i7-11370H.}
\label{fig:run_time}
\end{figure}

\end{document}